\documentclass{article} 
\usepackage{iclr2027_conference,times}

\usepackage{tabularx}

\usepackage{amsmath}
\usepackage{amssymb}
\usepackage{mathtools}
\usepackage{amsthm}

\usepackage{tikz}
\usetikzlibrary{automata, positioning, arrows, calc}

\usepackage{algorithm}
\usepackage{algorithmic}
\usepackage{thm-restate}
\usepackage{booktabs}
\usepackage{multirow}
\usepackage{enumitem}

\usepackage{subcaption}
\usepackage{xcolor}
\definecolor{negrew}{HTML}{aa605c}
\definecolor{unsafe}{HTML}{E5533D}
\colorlet{dom}{black!50}

\usepackage{xspace}

\theoremstyle{plain}

\newtheorem{proposition}{Proposition}
\newtheorem{lemma}{Lemma}
\theoremstyle{definition}
\newtheorem{example}{Example}

\usepackage{hyperref}
\usepackage{url}

\usepackage[capitalize]{cleveref}

\newif\ifdraft
\draftfalse
\usepackage{import}
\ifdraft
\usepackage[draft,nompar]{commenting}
\else
\usepackage[nompar]{commenting}
\fi

\declareauthor{lo}{lo}{red}
\authorcommand{lo}{comment}

\declareauthor{dw}{dw}{blue}
\authorcommand{dw}{comment}

\declareauthor{ak}{ak}{brown}
\authorcommand{ak}{comment}

\newcommand\alg{{SB-TRPO}\xspace}
\newcommand\algl{{Safety-Biased Trust Region Policy Optimisation}\xspace}

\newcommand\del{\Delta} 
\newcommand\dotp[2]{\langle #1,#2\rangle} 
\newcommand\conv\mu
\newcommand\epsc{\epsilon_{k}}
\newcommand\epskl\delta
\newcommand{\defeq}{\coloneqq}
\newcommand\fimp{\beta}

\newcommand\po{{\pi_k}}
\newcommand\pn{{\pi}}
\newcommand\parao{{\theta_{k}}}

\newcommand\scale\eta
\newcommand\divc{\kappa}

\newcommand\nat{\mathbb N}
\newcommand\supp{\mathrm{supp}}
\newcommand\dklm[2]{D_{\mathrm{KL}}^{\mathrm{max}}(#1\parallel #2)}
\newcommand\dklo[2]{\bar D_{\mathrm{KL}}(#1\parallel #2)}

\newcounter{problem}
\newcounter{update}

\newcommand{\problemtag}{%
  \refstepcounter{problem}%
  \tag{\textbf{Problem \theproblem}}
}
\newcommand\pref[1]{(\ref{#1})}

\newcommand{\updatetag}{%
  \refstepcounter{update}%
  \tag{\textbf{Update \theupdate}}
}

\title{SB-TRPO: Towards Safe Reinforcement\\Learning with Hard Constraints}

\author{Dominik Wagner\thanks{These authors contributed equally to this work.} \\
NTU Singapore \\
\texttt{dominik.wagner@ntu.edu.sg} \\
\And
Ankit Kanwar\footnotemark[1]\\
Sony Corporation
\\
\And
Luke Ong\\
NTU Singapore
}

\iclrfinalcopy 
\begin{document}

\maketitle

\begin{abstract}
In safety-critical domains, reinforcement learning (RL) systems must satisfy strict, zero-cost safety constraints while achieving meaningful task performance. Existing model-free methods can struggle to achieve high safety without substantially compromising task performance. We introduce \emph{Safety-Biased Trust Region Policy Optimisation (SB-TRPO)}, a principled approach to RL with zero-cost constraints, which requires only a fixed fraction of the maximal cost reduction achievable within the trust region, thus retaining flexibility for reward optimisation. We show that the idealised update nevertheless converges to zero cost and maximal reward amongst zero-cost policies in finite MDPs. A practical gradient-based approximation provides local improvements in both safety and reward under suitable gradient alignment. Experiments on \emph{Safety Gymnasium} demonstrate high safety alongside strong task performance.

\end{abstract}

\section{Introduction}
\label{sec:intro}

Reinforcement learning (RL) has achieved remarkable success in domains ranging from games to robotics, largely by optimising long-term cumulative rewards through trial-and-error interaction with an environment. However, in many real-world applications unconstrained reward maximisation is insufficient: agents must also satisfy safety or operational constraints. This problem is naturally formulated as a Constrained Markov Decision Process (CMDP, see \cite{cmdp}), which includes cost signals representing undesirable or unsafe outcomes in addition to rewards. Within this framework, safe RL has been studied under a wide range of objectives and algorithmic paradigms (see \citep{WachiSS24,DBLP:journals/pami/GuYDCWWK24} for surveys).

Most existing work formulates safe RL as maximising discounted expected reward subject to the discounted expected cost remaining below a positive threshold (e.g., \citep{Ray2019BenchmarkingSE,cpo,cup}). This formulation is appropriate when costs represent \emph{soft constraints}, such as resource consumption or wear-and-tear -- for example, an autonomous robot that must complete its mission before its battery is depleted. In such cases, incurring costs can be acceptable provided their expected frequency remains low.

However, in many safety-critical domains any non-zero cost corresponds to catastrophic failure. Robots operating near humans cannot afford collisions, and autonomous vehicles must never violate speed or distance limits in safety-critical zones. In these settings, expected-cost formulations with a positive cost threshold are misaligned with the true safety objective: policies must avoid safety violations entirely, rather than merely keeping them small in expectation.
 This motivates a \emph{hard-constrained formulation}, in which safety violations are unacceptable and the objective is to achieve zero cost with high reliability.

\paragraph{Limitations of Existing Approaches.}
Despite substantial progress in constrained RL, existing model-free methods struggle in hard-safety regimes \citep{belardinelli2025probabilistic}. Lagrangian approaches, such as TRPO- and PPO-Lagrangian \citep{Ray2019BenchmarkingSE}, rely on penalty multipliers to trade off reward and cost. In practice, this often leads to either unsafe policies (when penalties are small) or overly conservative behaviour with poor task completion (when penalties are large), with performance highly sensitive to hyperparameter choices.

Projection-based methods such as Constrained Policy Optimization (CPO) \citep{cpo} explicitly attempt to maintain feasibility at each update. While appealing in principle, this often results in severely restricted update directions or vanishing step sizes in challenging environments, leading to slow learning and low reward. More recent methods, including CRPO \citep{crpo} and PCRPO \cite{guetalaaai2024}, alternate between cost and reward optimisation or combine gradients via simple heuristics. When safety violations are frequent or hard to eliminate, these approaches tend to prioritise cost reduction almost exclusively, again resulting in poor task performance.

Overall, existing methods are primarily designed for \emph{soft} constraints and degrade qualitatively when applied to hard-safety settings: they either violate safety constraints or achieve safety at the cost of extremely low task performance. This gap motivates algorithms that are explicitly designed for hard-constrained reinforcement learning and can make reliable progress towards safety while still enabling meaningful reward optimisation.

\paragraph{Our Approach.}
We propose \emph{\algl (\alg)}, a principled approach to reinforcement learning with hard safety constraints. Its idealised update only requires a fixed fraction of the maximal achievable cost reduction within the trust region, controlled by a \emph{safety bias} parameter, while retaining flexibility for reward improvement.
Unlike CPO and related methods, \alg \emph{does not} switch into separate feasibility-recovery phases, avoiding over-conservatism and allowing reward improvement throughout training. 
In \alg, we approximate the idealised update using a dynamically computed convex combination of cost and reward TRPO directions.

\paragraph{Contributions.}
We summarise our main contributions:
\begin{itemize}[itemsep=0pt,topsep=0pt]
\item \textbf{Trust-Region Formulation for Zero-Cost Constraints.} We introduce an idealised policy update that controls the required cost reduction while retaining flexibility for reward optimisation, and develop a gradient-based approximation for its practical implementation.
\item \textbf{Asymptotic and Local Guarantees.} For the idealised update we show convergence to zero cost and maximal reward amongst zero-cost policies in finite MDPs. For the gradient-based approximation we instead establish local improvements in both safety and reward under suitable gradient alignment.
\item \textbf{Empirical Evaluation.} On a suite of \emph{Safety Gymnasium} tasks, SB-TRPO demonstrates strong task performance alongside high levels of safety.
\end{itemize}

\section{Related Work}
\label{sec:related}

\paragraph{Primal–Dual Methods.}
Primal–dual methods for constrained RL optimise a Lagrangian objective jointly over the policy and Lagrange multiplier to enforce cost constraints. RCPO \citep{rcpo} is an early two-timescale scheme using vanilla policy gradient and slow multiplier updates. Under additional assumptions (e.g.\ that all local optima are feasible), it converges to a locally optimal feasible policy. Likewise, TRPO-Lagrangian and PPO-Lagrangian \citep{Ray2019BenchmarkingSE} use trust-region or clipped-surrogate updates for the primal policy parameters. PID Lagrangian methods \citep{stooke2020responsivesafetyreinforcementlearning} augment the standard Lagrangian update with additional terms, reducing oscillations when overshooting the safety target is possible.
Similarly, APPO \citep{Dai_Ji_Yang_Zheng_Pan_2023}, which is based on PPO, augments the Lagrangian with a quadratic deviation term to dampen cost oscillations.

\paragraph{Trust-Region and Projection-Based Methods.}
CPO \citep{cpo} enforces a local approximation of the CMDP constraints within a trust region and performs pure cost-gradient steps when constraints are violated, aiming to always keep policy updates feasible.
FOCOPS
\citep{focops} is a first-order method solving almost the same abstract problem for policy updates in the nonparameterised policy space and then projects it back to the parameterised policy.
P3O \citep{ijcai2022p3o} is a PPO-based method motivated by CPO. It uses a Lagrange multiplier, increasing linearly to a fixed upper bound, and clipped surrogate updates to balance reward improvement with constraint satisfaction.
CUP \citep{cup}, in contrast to CPO, formulates surrogate reward and cost objectives using generalised performance bounds and GAEs, and projects each policy gradient update into the feasible set defined by these surrogates, allowing updates to jointly respect reward and cost approximations.
\cite{ctrpo} propose C\!-TRPO, which reshapes the geometry of the policy space by adding a barrier term to the KL divergence so trust regions contain only safe policies.  
Like CPO, C\!-TRPO switches to a dedicated recovery phase with pure cost-gradient updates when the policy becomes infeasible.  

\paragraph{Reward–Cost Switching.}
CRPO \citep{crpo} is a constrained RL method that alternates between reward maximisation and cost minimisation depending on whether the current cost estimate indicates the policy is infeasible. 
Motivated by CRPO’s tendency to oscillate between purely reward- and cost-focused updates, \cite{guetalaaai2024} propose PCRPO, which mitigates conflicts between reward and safety gradients using a softer switching mechanism and ideas from gradient manipulation \cite{pcgrad,cagrad}. 
When both objectives are optimised, the gradients are averaged, and if their angle exceeds $90^\circ$, each is projected onto the normal plane of the other.

\paragraph{Model-Based and Shielding Approaches.}
Shielding methods \citep{shielding,belardinelli2025probabilistic,jansen2020safe} intervene when a policy proposes unsafe actions, e.g.\ via lookahead or model predictive control, but typically require accurate dynamic models and can be costly. \citep{DBLP:conf/nips/Yu0Z22} train an auxiliary policy to edit unsafe actions, balancing reward against the extent of editing. Other typically model-based methods leverage control theory \citep{perkins2002lyapunov,berkenkamp2017safe,chow2018lyapunov,wang2023enforcing}. These approaches are complementary to our model-free approach.


\section{Problem Formulation}
\label{sec:background}
We consider a \emph{constrained Markov Decision Process} (CMDP) defined by the tuple 
$(\mathcal{S}, \mathcal{A}, P, r, c, \gamma)$, where $\mathcal{S}$ and $\mathcal{A}$ are 
the (potentially continuous) state and action spaces, $P(s' \mid s,a)$ is the transition kernel, 
$r: \mathcal{S} \to \mathbb{R}$ is the reward function, 
$c: \mathcal{S} \to \mathbb{R}_{\ge 0}$ is the cost function representing unsafe events, 
and $\gamma \in (0,1)$ is the discount factor. A stochastic policy $\pi(a \mid s)$ induces a 
distribution over trajectories $\tau = (s_0,a_0,s_1,a_1,\dots)$ together with $P$. \changed[dw]{Furthermore, we denote the standard discounted occupancy measure by $d_\pi$.}

In safety-critical domains, the objective is to maximise discounted expected reward while ensuring that unsafe states are never visited, i.e.\ $c(s_t)=0$ for all times $t$ almost surely:
\begin{align*}
\problemtag
\label{eq:hard}
\max_\pi \quad  &J_r(\pi) \defeq \mathbb{E}_{\tau\sim\pi} \left[ \sum_{t=0}^{\infty} \gamma^t\cdot r(s_t) \right]\qquad
\text{s.t.} \quad  &\mathbb P_{\tau\sim\pi} \left[ \exists t \in \mathbb{N}\ldotp c(s_t) > 0 \right] = 0
\end{align*}
\changed[dw]{Since costs are non-negative, any policy feasible under the following standard CMDP with expected discounted cost threshold of $0$ is also feasible 
for the original \pref{eq:hard}:}
\begin{align*}
\problemtag
\label{eq:relaxed}
\hspace{-1mm}\max_\pi\quad &J_r(\pi)\defeq\mathbb{E}_{\tau\sim\pi} \left[ \sum_{t=0}^{\infty} \gamma^t\cdot r(s_t) \right]
\quad\text{s.t.}\quad  J_c(\pi)\defeq
\mathbb{E}_{\tau\sim\pi} \left[ \sum_{t=0}^{\infty} \gamma^t \cdot c(s_t) \right] = 0
\end{align*}
This re-formulation enables the use of off-the-shelf constrained policy optimisation techniques 
whilst in principle still enforcing strictly safe behaviour. 

\paragraph{Discussion.}  
Most prior work in Safe RL focuses on \emph{positive cost thresholds} (e.g.\ 25 in \emph{Safety Gymnasium} \citep{ji2023safety}). However, in genuinely safety-critical applications there is no meaningful notion of an ``acceptable'' amount of catastrophic failure. Positive thresholds conflate the \emph{problem specification} with an \emph{algorithmic hyperparameter}, making training brittle, environment-dependent, and misaligned with true safety objectives (see \cref{sec:clim}).  

Our zero-cost formulation instead directly encodes the intrinsic hard-safety requirement. Conceptually, it provides a clear, principled target for algorithm design: policies should strictly avoid unsafe states, and algorithm design can focus on achieving this reliably rather than tuning a threshold that only approximates strict safety and explicitly permits a certain level of unsafety.

\changed[dw]{Finally, we focus on the safety of the \emph{learned deployment policy}, rather than safe exploration during training. We assume training takes place in simulation or a controlled environment where exploratory violations are safely contained; safe exploration is beyond our scope.}



\section{Method}
\label{sec:method}
We first derive the idealised SB-TRPO update, before presenting a practical approximation. Finally, we summarise the overall method in a complete algorithmic description.

\subsection{Safety-Biased Trust Region Update}
\label{sec:ideal}

In the setting of hard (zero‐cost) constraints, the idealised CPO \citep{cpo} update seeks to improve the current policy $\pi_k$ by maximising reward amongst feasible policies in the KL trust region:
\begin{align}
\label{eq:cpo}
\tag{CPO}
    \max_\pn \quad  J_r(\pn)
    \quad \text{s.t.} \quad
     &J_c(\pn)\leq 0,
     \qquad
     \dklo{\po}{\pn} \leq \epskl
\end{align}
\changed[dw]{where $\dklo\pi{\pi'}\defeq\mathbb E_{s\sim d_\pi}[D_{\mathrm{KL}}(\pi(\cdot\mid s)\parallel\pi'(\cdot\mid s))]$.}
When no such policy exists, CPO (and related methods such as C-TRPO) switches to a recovery step:
\begin{align}
\label{eq:recov}
\tag{Recovery}
    \min_\pn \quad  J_c(\pn)
    \quad \text{s.t.}
    \quad\dklo{\po}{\pn} \leq \epskl
\end{align}
This deliberately ignores task reward in favour of maximising cost reduction and can lead to reward-suboptimal zero-cost policies:
\begin{figure}
    \scalebox{0.8}{
\begin{minipage}{0.6\linewidth}
    \centering
        \begin{tikzpicture}[node distance = 1.7cm]
                \node[state, initial] (s0) {$s_0$};
                \node[state, below left of=s0] (s1) {$s_1$};
                \node[state, below right of=s0] (s2) {$s_3$};
                \node[state, left of=s1] (s3) {$s_2$};
                  \draw[->] 
                    (s0) edge node [below right,xshift=-4pt]{$a$} (s1)
                    (s0) edge node [below left,xshift=4pt]{$b$} (s2)
                    (s1) edge node [below]{$b$} (s2)
                    (s1) edge node [below]{$a$} (s3)
    
                    (s2) edge [loop right] node [right]{$\star$} (s2)
                    (s3) edge [loop left] node [left]{$\star$} (s3);
        \end{tikzpicture} 
\end{minipage}
\begin{minipage}{0.4\linewidth}
    \begin{align*}
        r(s_1)&\defeq 1,&&& r(s_i)&\defeq 0 &(i\neq 1)\\
        c(s_2)&\defeq 1,&&& c(s_j)&\defeq 0&(j\neq 2)
    \end{align*}
\end{minipage}
    }    
    \caption{Deterministic MDP for which CPO yields reward-suboptimal policies.}
    \label{fig:cpofail}
\end{figure}
\begin{example}
\label{ex:cpofail}
    Consider the MDP in \cref{fig:cpofail}. The policy
    $\pi^*(a\mid s_0)=1=\pi^*(b\mid s_1)$
    achieves
    $J_c(\pi^*)=0$ and $J_r(\pi^*)=\gamma$.
    Starting from an initial policy $\pi_0$ satisfying
    $\pi_0(a\mid s_0),\pi_0(a\mid s_1)\in(0,1)$, subsequent updates can never increase the probability of taking $a$ at $s_0$: doing so increases the probability of reaching $s_1$, where action $a$ incurs the cost with positive probability, whilst not decreasing the KL-divergence (see \cref{app:cpofail} for further details). Consequently,
    $\sup_{k\in\nat} J_r(\pi_k)
    \leq \gamma\cdot\pi_0(a\mid s_0)
    < J_r(\pi^*)$.
\end{example}
This limitation is also observed empirically: under zero-cost constraints, CPO often yields policies with low cost but poor task reward (cf.~\cref{sec:exp} and \citep{belardinelli2025probabilistic}).

To address this issue, we propose a more general update rule that seeks high reward while requiring only a controlled \emph{reduction} of cost: 
\begin{align*}
\updatetag
\label{eq:general}
    \max_\pn \quad  &J_r(\pn)
    \qquad 
    \text{s.t.}\quad
     &J_c(\pn)  \leq J_c(\po)-\epsc,\quad
     \dklo{\po}{\pn} \leq \epskl
\end{align*}
where for 
a \emph{safety bias} $\fimp \in (0,1]$, we define
\begin{align}
\label{eq:approxctr}
\epsc&\defeq\fimp\cdot (J_c(\po)-c^*_k)&\text{ where}\quad
    c^*_{k}&\defeq\min_\pi\; J_c(\pi)
    \quad \text{s.t.}\quad
    \dklo{\po}{\pi} \leq \epskl
\end{align}
Intuitively, $\epsc$ is a $\fimp$-fraction of the best achievable cost improvement inside the trust region.
This ensures that \pref{eq:general} always admits at least one solution, $\epsc \ge 0$, and avoids the need for an explicit recovery step.
 
 Importantly, \pref{eq:general} does \emph{not} correspond to solving a CMDP with a fixed relaxed cost threshold $J_c(\pi)\le d$, see also \cref{thm:optideal}: the required cost decrease is dynamically adjusted according to the current policy via \cref{eq:approxctr}, whilst
 the parameter $\fimp$ controls how aggressively cost is reduced.
 
 \paragraph{$\fimp=1$ Recovers CPO.}  Setting $\fimp = 1$ forces each update to pursue \emph{maximal} cost reduction within the trust region, simplifying the cost constraint in \pref{eq:general} to $J_c(\pn)\leq c^*_{k}$. Note that $c_{k}^*=0$ iff \pref{eq:cpo} is feasible. Thus, \pref{eq:general} for $\fimp=1$ elegantly \emph{captures both the recovery and standard feasible phases of CPO} (breaking ties during recovery favourably for reward).  

\paragraph{Intuition for Using a Safety Bias $\fimp < 1$.}
 \changed[dw]{
Choosing $\fimp < 1$ intentionally relaxes this requirement: the policy must still reduce cost, but the available policy change can be allocated more flexibly rather than being used entirely to maximise the immediate cost reduction. In \cref{ex:cpofail}, this leaves part of the KL budget available for reducing the probability of the unsafe action $a$ in $s_1$, rather than forcing the update to reduce the probability of reaching $s_1$ in the first place.
}

\paragraph{Asymptotic Optimality.} Despite requiring only a fraction of the maximal immediate cost improvement, \pref{eq:general} yields asymptotic convergence to the  minimal cost (for $\fimp\in(0,1]$) and to the corresponding optimal reward (for $\fimp\in(0,1)$) in finite MDPs:

\begin{restatable}[Asymptotic Optimality for \pref{eq:general}]{theorem}{optideal}
\label{thm:optideal}
    Assume the MDP is finite.
    Let $\pi_k$ be obtained via \pref{eq:general} for arbitrary $\pi_0$.
    Define $J_c^*\defeq\min_\pi J_c(\pi)$ and $J_r^*\defeq\max_{\pi:J_c(\pi)=J_c^*} J_r(\pi)$. 
    \begin{enumerate}[noitemsep,topsep=0pt]
        \item If $\fimp\in(0,1]$ then $\lim_{k\to\infty} J_c(\pi_k)=J_c^*$ (and the limit exists).
        \item If $\fimp\in(0,1)$ then also $\lim_{k\to\infty} J_r(\pi_k)= J_r^*$ (and the limit exists).
    \end{enumerate}
\end{restatable}
In particular, when zero-cost policies exist, the theorem implies that $J_c(\pi_k)\to0$.

\changed[dw]{
\begin{proof}[Proof sketch]
    Full details are provided in \cref{app:optideal}.
    The cost sequence is monotonically non-increasing by construction. If an accumulation point were cost-suboptimal, it would admit a cost-improving policy strictly within the $\delta$-KL trust region, which eventually also lies within the trust region around $\pi_k$ along the corresponding subsequence. Thus, eventually even the required fractional improvement would yield a strict improvement over the accumulation point.

    For the reward result, consider a subsequence approaching an accumulation point with suboptimal reward. Mixing this accumulation point with a reward-optimal cost-optimal policy yields a policy $\pi'$ that remains cost-optimal while lying strictly within the KL trust region. We can then consider policies mixing each iterate with an $\fimp$-fraction of $\pi'$, which eventually become feasible (since $\fimp<1$) for the update and yield a strict reward improvement. 
\end{proof}
}

\subsection{Approximate Solution}
\label{sec:approx}

Since direct evaluation of $J_r(\pi)$ and $J_c(\pi)$ for arbitrary $\pi$ is infeasible, we employ the standard \emph{surrogate objectives} \citep{TRPO} with respect to a reference policy $\po$.  
For reward, 
\begin{align*}
    \mathcal L_{r,\po}(\pi) \defeq 
    J_r(\po)+  \mathbb E_{s,a\sim d_{\po}}\!\left[
    \frac{\pi(a \mid s)}{\po(a \mid s)}\cdot A_r^{\po}(s,a)
    \right]
\end{align*}
\changed[dw]{where $A_r^{\po}$ denotes the advantage function. The Policy Improvement Bound \citep{TRPO,li2022analytical} bounds $|\mathcal L_{r,\po}(\pi)-J_r(\pi)|$ in terms of the KL-divergence between $\po$ and $\pi$ (analogously for cost). In particular, $J_r(\po)=\mathcal L_{r,\po}(\po)$.
Hence, we may approximate \pref{eq:general}:}
\begin{align}
\updatetag
\label{eq:ideal}
    \max_\pn \quad &\mathcal L_{r,\po}(\pn)\qquad
    \text{s.t.} \quad
    &\mathcal L_{c,\po}(\pn)  \leq \mathcal L_{c,\po}(\po)-\epsc,\qquad
 \dklo{\po}{\pn} \leq \epskl
    \notag
\end{align}
As in TRPO, the KL constraint guarantees that the surrogates remain close to the true performances.

\paragraph{Second-Order Approximation.}
Henceforth, we assume differentiably parameterised policies $\pi_\theta$ and overload notation, e.g. $J_c(\theta)$ for $J_c(\pi_\theta)$. Linearising the reward and cost objectives around the current parameters $\parao$ and approximating the KL divergence by a quadratic form with the Fisher information matrix  $F$
yields the quadratic program
\begin{align}
\updatetag
\label{eq:approx}
\max_{\del \in \mathbb R^d} \; \dotp{g_r} \del
\quad \text{s.t.} \quad
\dotp{g_c} \del \leq -\epsc,
\quad \tfrac{1}{2}\cdot\del^\top\cdot F\cdot \del \leq \epskl,
\end{align}
where $g_r\defeq \nabla\mathcal L_{r,\parao}(\parao)$ and $g_c\defeq\nabla\mathcal L_{c,\parao}(\parao)$.
To choose $\epsc$ and approximate \cref{eq:approxctr}, we analogously approximate the trust-region step that maximally decreases the cost surrogate
\begin{align}
        \label{eq:approxc}
    \del_c &\defeq \arg\min_{\del} \; \dotp{g_c} \del
        \quad \text{s.t.} \quad \tfrac{1}{2}\cdot \del^\top\cdot F\cdot \del \le \epskl.
\end{align}
and set $\epsc \defeq -\fimp\cdot\dotp{g_c}{\del_c}$ with safety-bias $\fimp\in(0,1]$.
This guarantees feasibility of \eqref{eq:approx}.
 The new parameters are then updated via $\theta_{k+1}\defeq \parao + \del^*$, with $\del^*$ the solution of \pref{eq:approx}, analogous to the TRPO step but enforcing an explicit local cost reduction.

\paragraph{Approximate Solution via Cost-Biased Convex Combination.}
Solving \pref{eq:approx} analytically requires computing the natural gradients $F^{-1}g_r$ and $F^{-1}g_c$ and their coefficients (cf.~\cref{prop:analytic} in \cref{app:approxi}), which is numerically fragile for large or ill-conditioned $F$.  
Instead, we follow TRPO and compute the KL-constrained reward and cost steps separately using the conjugate gradient method \citep{CG}:
\begin{align}
\label{eq:approxr}
    \del_r &\defeq \arg\max_{\del} \; \dotp{g_r} \del
        \quad \text{s.t.} \quad \tfrac{1}{2} \cdot\del^\top\cdot F\cdot \del \le \epskl
\end{align}
with $\del_c$ given by \cref{eq:approxc}.  
Both steps satisfy the same KL constraint, and thus any convex combination
$\del \defeq (1-\conv)\cdot \del_r + \conv\cdot \del_c$, where \(\conv\in[0,1]\), also satisfies it (see \cref{lem:kl}).
We therefore choose the best reward-improving combination which also satisfies $\dotp{g_c}\del\leq-\epsc$ (see \cref{fig:conv} for a visualisation and \cref{lem:conv}):
\begin{align}
\label{eq:conv}
    \hspace*{-13mm}\conv \defeq 
    \begin{cases}
        \frac{\dotp{g_c} {\del_r} + \epsc}{\dotp{g_c}{\del_r}-\dotp{g_c}{\del_c}} & \text{if $\dotp{g_c}{\del_c}\neq \dotp{g_c} {\del_r}$ and $\dotp{g_c} {\del_r}\geq-\epsc$}\\
        0 & \text{otherwise}
    \end{cases}
\end{align}

\paragraph{Performance Improvement Guarantee.}
This approximation of \pref{eq:approx} ensures an approximate cost decreases at each update:
\begin{align*}
    J_c(\theta+\del)-J_c(\theta) \approx \dotp {g_c}{\del}\leq \fimp\cdot\dotp{g_c}{\del_c}\leq 0
\end{align*}
Likewise, if the reward gradient is non-zero and its angle with $\del_c$ does not exceed $90^\circ$, an increase in reward is guaranteed. Formally, sufficiently small steps along 
$\del$ are required for the local gradient approximations to hold:
\begin{restatable}[Performance Improvement]{theorem}{perf}
\label{thm:perf}
Let $\del \defeq (1-\conv)\cdot \del_r + \conv\cdot \del_c$ for $\del_r$, $\del_c$ and $\conv$ defined in \cref{eq:conv,eq:approxr,eq:approxc}, respectively.
There exists $\scale\in(0,1]$ such that
\begin{enumerate}[noitemsep,topsep=0pt]
    \item if $g_c\neq0$ then $J_c(\theta+\scale\cdot\del)<J_c(\theta)$
    \item if \changed[dw]{$\fimp<1$}, $g_r\neq 0$ and $\dotp{g_r}{\del_c}\geq 0$ (in particular, $g_c=0$) then $J_r(\theta+\scale\cdot\del)\geq J_r(\theta)$.  
\end{enumerate}
\end{restatable}
This result complements \cref{thm:optideal} with guarantees for the gradient-based approximate solution of \pref{eq:general}.
Unlike CPO and related methods such as C-TRPO \citep{ctrpo}, which have a separate feasibility recovery phase, our update can guarantee reward improvement for all steps in which gradients are well-aligned.
Once the cost gradient is zero, updates focus on reward improvement (see also \cref{fig:cost_imp}). Overall, this scheme dynamically balances cost and reward objectives, naturally stabilising learning.

\newcommand\plotfig[5]{
    \begin{tikzpicture}[scale=1.0,>=stealth]
    
    
    \def\rx{#2}
    \def\ry{#3}
    \def\cx{#4}
    \def\cy{#5}
    
    \coordinate (O) at (0,0);
    \coordinate (r) at (\rx,\ry);
    \coordinate (c) at (\cx,\cy);
    
    \def\cr{(-\cx*\rx-\cy*\ry)}
    \def\cc{(-\cx*\cx-\cy*\cy)}
    \pgfmathsetmacro\conve{min(1,max(0, (\cr-#1*\cc)/(\cr-\cc+0.0001) ))}
    \pgfmathsetmacro{\converound}{round(100*\conve)/100} 
    \pgfmathsetmacro\epsi{-#1*\cc}
    
    \coordinate (dir) at ($ (r)! \conve ! (c) $);
    
    \draw[->,very thick,blue!80!black] (O) -- (dir) node[midway, left] {$\del$};
    \draw[->,densely dotted,thick,red!80!black] (O) -- (r) node[midway, below] {$\del_r$};
    \draw[->,densely dashed,thick,green!70!black] (O) -- (c) node[midway, left] {$\del_c$};
    \filldraw[black] (O) circle (1pt) node[below left] {};
    
    \node (lab) at (0.3,-0.4) {\small$\conv=\pgfmathprintnumber[fixed,precision=2]{\conve}$};

    \end{tikzpicture}
}

\begin{figure}[t]
    \centering
    \hfill
    \begin{subfigure}[c]{0.23\linewidth}
        \plotfig{0.7}{1.5}{0.7}{-1}{1}
    \end{subfigure}
    \hfill
    \begin{subfigure}[c]{0.23\linewidth}
        \plotfig{0.7}{1.5}{0.4}{1}{1}
    \end{subfigure}
    \hfill
    \begin{subfigure}[c]{0.23\linewidth}
        \plotfig{0.7}{1}{.3}{1}{1}
    \end{subfigure}
    \hfill
    \begin{subfigure}[c]{0.23\linewidth}
        \plotfig{0.7}{1.5}{0.5}{-1}{-1}
    \end{subfigure}
    \hfill
    
    \caption{Visualisation of the adaptive convex combination $\del$ of $\del_r$ and $\del_c$ given by \cref{eq:conv} for $\epsc\defeq 1.4=-\fimp\cdot\dotp {g_c}{\del_c}$, where $\fimp\defeq0.7$, and the special case that $\del_r=g_r$ and $\del_c=-g_c$.}
    \label{fig:conv}
\end{figure}

\subsection{Practical Algorithm}
\label{sec:imp}

Our practical algorithm, implementing the ideas of the preceding subsections, is given in \cref{alg}.
We now highlight the key additions.

\begin{algorithm}[t]
    \caption{\algl (\alg)}
    \label{alg}
    \begin{algorithmic}[1] 
    \REQUIRE KL divergence limit $\epskl>0$, safety bias $\fimp\in[0,1]$, training epochs $N\in\mathbb N$ 
    \STATE initialise $\theta$
    \STATE $\divc\leftarrow 10^{-8}$ \COMMENT{small constant to avoid division by $0$}
        \FOR{$N$ epochs}
        \STATE $\mathcal D\leftarrow$ collect rollouts under policy $\pi_\theta$
        \STATE $g_r,g_c\leftarrow $ policy gradient estimate of $\nabla\mathcal L_{r,\pi_\theta}(\pi_\theta)$ and $\nabla\mathcal L_{c,\pi_\theta}(\pi_\theta)$ using $\mathcal D$
        \STATE $\del_r,\del_c\leftarrow$ calculated via the conjugate gradient algorithm to optimise \cref{eq:approxr,eq:approxc} 
        \STATE $\mu \leftarrow \max \left\{0, \frac{\dotp{g_c} {\del_r} - \fimp\cdot\dotp {g_c}{\del_c}}{\dotp{g_c}{\del_r} - \dotp{g_c}{\del_c}+\divc}\right\}$ \COMMENT{substituting definition of $\epsc$ in \cref{eq:conv}}

        \STATE $\del\leftarrow(1-\conv)\cdot \del_r + \conv\cdot \del_c$

        \STATE $\scale \leftarrow 1$ \COMMENT{line search for constraint satisfaction}
\REPEAT
  \STATE decrease $\scale$
\UNTIL{$\widehat {D_{KL}}[\pi_\theta \parallel \pi_{\theta+\scale\cdot\del}] \leq \epskl$ and $\widehat{\mathcal L_{c,\pi_\theta}}(\pi_{\theta+\scale\cdot\del}) \leq \widehat{\mathcal L_{c,\pi_\theta}}(\pi_\theta)$ 
}

        \STATE $\theta \leftarrow\theta + \scale\cdot\del$
        \ENDFOR
    \end{algorithmic}
\label{alg: CTRPO}
\end{algorithm}

\paragraph{Line Search.}
In practice, we follow standard TRPO procedures by performing a line search along the update direction $\del$: the KL constraint is enforced empirically to ensure that the surrogate approximations of $J_r$ and $J_c$ remain accurate; and the empirical reduction in the surrogate cost loss is checked to operationalise the cost decrease guarantee of \cref{thm:perf}, scaling the step if necessary.

\paragraph{Gradient Estimation.} 
 The gradients of the surrogate objective 
 $g_r=\nabla\mathcal L_{r,\parao}(\parao)$ (analogously for cost) coincide with $\nabla J_r(\parao)$, which is Monte Carlo estimated using the policy gradient theorem. The Fisher information matrix $F$ is similarly estimated from trajectories. 

\section{Empirical Evaluation}
\label{sec:exp}
\subsection{Experimental Setup}
\label{sec:setup}

\paragraph{Benchmarks.}


\emph{Safety Gymnasium} provides a broad suite of environments for benchmarking safe RL algorithms.
Originally introduced by OpenAI \citep{Ray2019BenchmarkingSE}, it has since been maintained and extended by \cite{ji2023safety}, who also developed \emph{SafePO}, a library of safe RL algorithms.
We focus on two classes of environments: \emph{Safe Navigation} and \emph{Safe Velocity}. 
In Safe Navigation, a robot equipped with LIDAR sensors must complete tasks while avoiding unsafe hazards. 
We use the \emph{Point} and \emph{Car} robots on tasks \emph{Circle}, \emph{Push} and \emph{Goal} at \emph{level 2}, which is most challenging due to the greatest number of hazards. 
In contrast, \emph{Safe Velocity} environments test adherence to velocity limits, and we use the \emph{Hopper} and \emph{Swimmer} robots adapted from the MuJoCo locomotion suite \citep{todorov2012mujoco}. 
The state spaces are up to 88-dimensional; see \cref{app:sge} for more details.

\paragraph{Baselines.}
We compare our approach against a wide range of  state-of-the-art implementations from SafePO \citep{ji2023safety} of model-free constrained RL algorithms: TRPO- and PPO-Lagrangian \citep{Ray2019BenchmarkingSE}, CPO \citep{cpo}, CUP \citep{cup}, CPPO-PID \citep{stooke2020responsivesafetyreinforcementlearning}, FOCOPS \citep{focops}, RCPO \citep{rcpo}, PCPO \citep{pcpo}, P3O \citep{ijcai2022p3o} and C-TRPO \citep{ctrpo}.\footnote{\label{footnote:p3o-appo-crpo}Code of reward–cost switching methods such as (P)CRPO \citep{crpo,guetalaaai2024} and APPO \citep{Dai_Ji_Yang_Zheng_Pan_2023} is not publicly available \citep{DBLP:journals/pami/GuYDCWWK24}.}
All baselines require a cost tolerance threshold, which we set to $0$ to target almost-sure safety (cf.~\pref{eq:relaxed}).

\paragraph{Implementation Details.}
Following the default setup in SafePO, we use a homoscedastic Gaussian policy with state-independent variance, the mean of which is parameterised by a feedforward network with two hidden layers of 64 units.
Besides, we use SafePO's defaults for existing hyperparameter (cf.~\cref{app:ht}), and \emph{the same safety-bias $\fimp=0.75$  across all tasks} (cf.~ \cref{sec:ablation}).

We consider two variants of SB-TRPO: one using Monte Carlo advantage estimates and one using Generalized Advantage Estimation (GAE) with a learned critic. We compare them in \cref{sec:ablation} and use the Monte Carlo variant for the main results in \cref{tab:main-results}.

\paragraph{Training Details.}
We train the policies for $2 \times 10^7$ time steps ($10^3$ epochs), running $2 \times 10^4$ steps per epoch with $20$ vectorised environments per seed. 
All results are reported as averages over 5 seeded parameter initialisations, with standard deviations indicating variability across seeds. 

\paragraph{Metrics.}
We compare our performance with the baselines using the standard metrics of \emph{rewards} and \emph{costs}, 
averaged over the past 50 episodes during training. 

To capture safety and task performance under hard constraints (as formalised in \pref{eq:hard}), we introduce two additional metrics: 
\emph{safety probability} is the fraction of episodes completed without any safety violation, i.e.\ with zero cost;
\emph{safe reward} is the average return over all episodes, with returns from episodes with any safety violation set to $0$ (see \cref{eq:defmetric} in \cref{app:metric}).

\subsection{Results}
\label{sec:res}

\begin{table*}[t!]
    \centering
    \caption{Final metrics across level-2 \emph{Safety Gymnasium} tasks. Red: negative (safe) reward, italics: safety probability $< 50\%$, grey: Pareto-dominated pairs, bold: best value per metric (only among methods with positive (safe) reward and safety probability $\geq 50\%$ for the other metric).}
    \label{tab:resall}
    \resizebox{\textwidth}{!}{%
    \begin{tabular}{l*{8}{cc}}

\toprule
\multirow{2}{*}{Method} & \multicolumn{2}{c}{Car Circle} & \multicolumn{2}{c}{Point Push} & \multicolumn{2}{c}{Car Push} & \multicolumn{2}{c}{Point Goal} & \multicolumn{2}{c}{Car Goal} & \multicolumn{2}{c}{Hopper Velocity} & \multicolumn{2}{c}{Swimmer Velocity} \\
\cmidrule(lr){2-15}
 & Safe Rew.$^{\uparrow}$ & Safe Prob.$^{\uparrow}$ 
 & Safe Rew.$^{\uparrow}$ & Safe Prob.$^{\uparrow}$ 
 & Safe Rew.$^{\uparrow}$ & Safe Prob.$^{\uparrow}$ 
 & Safe Rew.$^{\uparrow}$ & Safe Prob.$^{\uparrow}$ 
 & Safe Rew.$^{\uparrow}$ & Safe Prob.$^{\uparrow}$ 
 & Safe Rew.$^{\uparrow}$ & Safe Prob.$^{\uparrow}$ 
 & Safe Rew.$^{\uparrow}$ & Safe Prob.$^{\uparrow}$ \\
\midrule

Ours & $7.5\pm1.5$ & $\boldsymbol{0.99\pm0.01}$ & $\boldsymbol{0.33\pm0.18}$ & $\boldsymbol{0.79\pm0.08}$ & $\boldsymbol{0.38\pm0.12}$ & $0.71\pm0.07$ & $\boldsymbol{1.6\pm0.5}$ & $0.68\pm0.07$ & $\boldsymbol{1.4\pm0.4}$ & $0.66\pm0.08$ & \color{dom}$860\pm378$ & \color{dom}$\boldsymbol{1.0\pm0.0}$ & $48\pm26$ & $\boldsymbol{0.98\pm0.02}$ \\
PPO-Lag & \color{dom}$7.5\pm2.4$ & \color{dom}$0.86\pm0.17$ & \color{negrew}$-0.29\pm0.8$ & \color{negrew}$0.59\pm0.22$ & \color{negrew}$-0.13\pm0.34$ & \color{negrew}$0.65\pm0.1$ & \color{dom}$\mathit{1.3\pm1.1}$ & \color{dom}$\mathit{0.47\pm0.13}$ & \color{dom}$\mathit{1.2\pm1.2}$ & \color{dom}$\mathit{0.45\pm0.16}$ & \color{dom}$510\pm500$ & \color{dom}$0.58\pm0.44$ & \color{dom}$18\pm38$ & \color{dom}$0.57\pm0.42$ \\
TRPO-Lag & \color{dom}$\boldsymbol{11\pm4}$ & \color{dom}$0.81\pm0.26$ & \color{negrew}$-0.49\pm0.45$ & \color{negrew}$0.78\pm0.16$ & \color{negrew}$-0.19\pm0.25$ & \color{negrew}$0.89\pm0.04$ & $0.09\pm0.29$ & $\boldsymbol{0.79\pm0.06}$ & $0.23\pm0.38$ & $0.75\pm0.07$ & \color{dom}$72\pm90$ & \color{dom}$0.6\pm0.5$ & \color{dom}$\mathit{1.3\pm5.1}$ & \color{dom}$\mathit{0.33\pm0.42}$ \\
CPO & \color{dom}$0.96\pm0.76$ & \color{dom}$\boldsymbol{0.99\pm0.03}$ & \color{negrew}$-1.2\pm0.9$ & \color{negrew}$0.77\pm0.19$ & \color{negrew}$-1.3\pm0.3$ & \color{negrew}$0.93\pm0.05$ & \color{negrew}$-1.3\pm0.4$ & \color{negrew}$0.86\pm0.05$ & \color{negrew}$-0.73\pm0.23$ & \color{negrew}$0.82\pm0.09$ & \color{dom}$880\pm182$ & \color{dom}$\boldsymbol{1.0\pm0.0}$ & \color{negrew}$-3.8\pm9.3$ & \color{negrew}$0.99\pm0.03$ \\
CUP & \color{dom}$7.1\pm3.5$ & \color{dom}$0.75\pm0.28$ & \color{dom}$0.12\pm0.19$ & \color{dom}$0.67\pm0.19$ & \color{negrew}$\mathit{-0.0062\pm0.1729}$ & \color{negrew}$\mathit{0.47\pm0.19}$ & \color{dom}$\mathit{0.95\pm0.74}$ & \color{dom}$\mathit{0.15\pm0.09}$ & \color{dom}$\mathit{0.93\pm0.74}$ & \color{dom}$\mathit{0.15\pm0.1}$ & $\boldsymbol{1000\pm289}$ & $0.72\pm0.19$ & \color{dom}$34\pm52$ & \color{dom}$0.9\pm0.1$ \\
CPPO-PID & \color{dom}$0.96\pm1.85$ & \color{dom}$0.91\pm0.12$ & \color{negrew}$-2.7\pm2.7$ & \color{negrew}$0.79\pm0.11$ & \color{negrew}$-0.71\pm0.58$ & \color{negrew}$0.72\pm0.08$ & \color{negrew}$-1.8\pm0.7$ & \color{negrew}$0.81\pm0.07$ & \color{negrew}$-0.59\pm0.5$ & \color{negrew}$0.69\pm0.08$ & \color{dom}$860\pm263$ & \color{dom}$0.84\pm0.16$ & \color{negrew}$-0.33\pm9.93$ & \color{negrew}$0.95\pm0.03$ \\
FOCOPS & $8.1\pm2.0$ & $0.96\pm0.09$ & \color{negrew}$-0.2\pm0.5$ & \color{negrew}$0.66\pm0.3$ & \color{dom}$0.05\pm0.251$ & \color{dom}$0.72\pm0.08$ & $\mathit{2.2\pm1.2}$ & $\mathit{0.34\pm0.05}$ & $\mathit{1.6\pm2.3}$ & $\mathit{0.27\pm0.3}$ & $970\pm307$ & $\boldsymbol{1.0\pm0.0}$ & \color{dom}$15\pm10$ & \color{dom}$0.88\pm0.09$ \\
RCPO & $10\pm3$ & $0.9\pm0.18$ & \color{negrew}$-0.48\pm0.45$ & \color{negrew}$0.84\pm0.08$ & \color{negrew}$-0.48\pm0.4$ & \color{negrew}$0.81\pm0.07$ & \color{negrew}$-0.012\pm0.366$ & \color{negrew}$0.85\pm0.06$ & $0.21\pm0.4$ & $\boldsymbol{0.77\pm0.06}$ & \color{dom}$420\pm364$ & \color{dom}$\boldsymbol{1.0\pm0.0}$ & \color{dom}$7.0\pm12.3$ & \color{dom}$0.67\pm0.25$ \\
PCPO & $\boldsymbol{11\pm2}$ & $0.89\pm0.17$ & \color{negrew}$-1.1\pm0.7$ & \color{negrew}$0.79\pm0.11$ & \color{negrew}$-0.48\pm0.38$ & \color{negrew}$0.82\pm0.06$ & \color{negrew}$-1.3\pm0.6$ & \color{negrew}$0.81\pm0.08$ & \color{negrew}$-0.75\pm0.26$ & \color{negrew}$0.77\pm0.08$ & \color{dom}$460\pm501$ & \color{dom}$0.72\pm0.36$ & \color{dom}$\mathit{5\pm10}$ & \color{dom}$\mathit{0.13\pm0.26}$ \\
P3O & \color{dom}$7.1\pm2.4$ & \color{dom}$0.94\pm0.07$ & \color{dom}$0.23\pm0.12$ & \color{dom}$0.78\pm0.1$ & $0.29\pm0.13$ & $\boldsymbol{0.76\pm0.06}$ & \color{negrew}$-0.024\pm0.521$ & \color{negrew}$0.64\pm0.06$ & \color{dom}$0.36\pm0.29$ & \color{dom}$0.52\pm0.11$ & \color{dom}$110.0\pm78$ & \color{dom}$0.98\pm0.017$ & $\boldsymbol{94\pm51}$ & $0.88\pm0.08$ \\
C-TRPO & \color{dom}$0.54\pm0.5$ & \color{dom}$0.97\pm0.05$ & \color{negrew}$-1.4\pm1.2$ & \color{negrew}$0.83\pm0.09$ & \color{negrew}$-1.2\pm0.5$ & \color{negrew}$0.88\pm0.05$ & \color{negrew}$-2.2\pm1.1$ & \color{negrew}$0.88\pm0.06$ & \color{negrew}$-1.9\pm0.5$ & \color{negrew}$0.89\pm0.07$ & \color{dom}$4.2\pm1.7$ & \color{dom}$\boldsymbol{1\pm0}$ & \color{negrew}$-18\pm8$ & \color{negrew}$0.97\pm0.03$ \\
\bottomrule

\\\\
\toprule
        \multirow{2}{*}{Method} & \multicolumn{2}{c}{Car Circle} & \multicolumn{2}{c}{Point Push} & \multicolumn{2}{c}{Car Push} & \multicolumn{2}{c}{Point Goal} & \multicolumn{2}{c}{Car Goal} & \multicolumn{2}{c}{Hopper Velocity} & \multicolumn{2}{c}{Swimmer Velocity} \\
        \cmidrule(lr){2-15}
         & Rewards$^{\uparrow}$ & Costs$^{\downarrow}$ 
& Rewards$^{\uparrow}$ & Costs$^{\downarrow}$ 
& Rewards$^{\uparrow}$ & Costs$^{\downarrow}$ 
& Rewards$^{\uparrow}$ & Costs$^{\downarrow}$ 
& Rewards$^{\uparrow}$ & Costs$^{\downarrow}$ 
& Rewards$^{\uparrow}$ & Costs$^{\downarrow}$ 
& Rewards$^{\uparrow}$ & Costs$^{\downarrow}$ \\
        \midrule

Ours & $7.5\pm1.5$ & $\boldsymbol{0.28\pm0.66}$ & $\boldsymbol{0.43\pm0.21}$ & $\boldsymbol{8.0\pm6.5}$ & $\boldsymbol{0.53\pm0.16}$ & $\boldsymbol{14\pm7}$ & $\boldsymbol{2.4\pm0.6}$ & $18\pm6$ & $\boldsymbol{2.1\pm0.6}$ & $18\pm6$ & \color{dom}$860\pm377$ & \color{dom}$0.00028\pm0.00208$ & \color{dom}$49\pm27$ & \color{dom}$0.87\pm3.24$ \\
PPO-Lag & \color{dom}$8.7\pm2.0$ & \color{dom}$22\pm37$ & \color{negrew}$-0.47\pm1.31$ & \color{negrew}$27\pm21$ & \color{negrew}$-0.2\pm0.52$ & \color{negrew}$34\pm20$ & $\mathit{2.8\pm2.2}$ & $\mathit{42\pm18}$ & $\mathit{2.6\pm2.5}$ & $\mathit{54\pm29}$ & \color{dom}$890\pm507$ & \color{dom}$2.9\pm2.8$ & \color{dom}$40\pm60$ & \color{dom}$6.5\pm7.2$ \\
TRPO-Lag & $\boldsymbol{13\pm1}$ & $18\pm33$ & \color{negrew}$-0.65\pm0.55$ & \color{negrew}$9.3\pm14.5$ & \color{negrew}$-0.21\pm0.27$ & \color{negrew}$6.3\pm4.6$ & $0.11\pm0.35$ & $\boldsymbol{17\pm9}$ & $0.29\pm0.5$ & $\boldsymbol{16\pm7}$ & \color{dom}$160\pm77$ & \color{dom}$2.1\pm2.6$ & \color{dom}$\mathit{21\pm16}$ & \color{dom}$\mathit{6.3\pm6.0}$ \\
CPO & \color{dom}$0.96\pm0.75$ & \color{dom}$0.98\pm6.63$ & \color{negrew}$-1.6\pm1.1$ & \color{negrew}$5.8\pm9.6$ & \color{negrew}$-1.4\pm0.3$ & \color{negrew}$4.2\pm4.4$ & \color{negrew}$-1.5\pm0.5$ & \color{negrew}$9.3\pm8.0$ & \color{negrew}$-0.89\pm0.27$ & \color{negrew}$13\pm8$ & $880\pm185$ & $\boldsymbol{0.0\pm0.0}$ & \color{negrew}$-3.8\pm9.4$ & \color{negrew}$0.016\pm0.037$ \\
CUP & \color{dom}$9.6\pm2.5$ & \color{dom}$48\pm86$ & \color{dom}$0.18\pm0.26$ & \color{dom}$25\pm22$ & \color{negrew}$\mathit{-0.026\pm0.362}$ & \color{negrew}$\mathit{85\pm50}$ & $\mathit{7.2\pm2.8}$ & $\mathit{110\pm28}$ & $\mathit{6.1\pm2.8}$ & $\mathit{110\pm31}$ & $\boldsymbol{1400\pm114}$ & $0.43\pm0.35$ & $38\pm58$ & $0.14\pm0.18$ \\
CPPO-PID & \color{dom}$1.3\pm1.8$ & \color{dom}$27\pm55$ & \color{negrew}$-3.4\pm3.3$ & \color{negrew}$13\pm17$ & \color{negrew}$-0.99\pm0.78$ & \color{negrew}$34\pm28$ & \color{negrew}$-2.3\pm0.9$ & \color{negrew}$17\pm10$ & \color{negrew}$-0.79\pm0.69$ & \color{negrew}$23\pm10$ & $1000\pm228$ & $0.081\pm0.107$ & $0.71\pm10.89$ & $\boldsymbol{0.044\pm0.081}$ \\
FOCOPS & $8.4\pm1.9$ & $3.9\pm18.2$ & \color{negrew}$-0.29\pm0.76$ & \color{negrew}$35\pm61$ & \color{dom}$0.069\pm0.338$ & \color{dom}$35\pm24$ & $\mathit{6.5\pm3.3}$ & $\mathit{88\pm37}$ & $\mathit{6.0\pm5.3}$ & $\mathit{86\pm49}$ & \color{dom}$970\pm309$ & \color{dom}$0.093\pm0.135$ & \color{dom}$20\pm12$ & \color{dom}$3.0\pm5.7$ \\
RCPO & \color{dom}$11\pm2$ & \color{dom}$13\pm35$ & \color{negrew}$-0.57\pm0.51$ & \color{negrew}$6.3\pm8.8$ & \color{negrew}$-0.6\pm0.49$ & \color{negrew}$19\pm20$ & \color{negrew}$-0.00026\pm0.41018$ & \color{negrew}$12\pm7$ & $0.29\pm0.5$ & $\boldsymbol{16\pm7}$ & \color{dom}$470\pm361$ & \color{dom}$0.97\pm1.93$ & \color{dom}$19\pm18$ & \color{dom}$5.1\pm6.7$ \\
PCPO & $12\pm1$ & $4.4\pm8.5$ & \color{negrew}$-1.4\pm0.8$ & \color{negrew}$14\pm23$ & \color{negrew}$-0.58\pm0.46$ & \color{negrew}$17\pm18$ & \color{negrew}$-1.6\pm0.7$ & \color{negrew}$18\pm13$ & \color{negrew}$-0.95\pm0.3$ & \color{negrew}$17\pm9$ & \color{dom}$750\pm589$ & \color{dom}$3.9\pm4.9$ & \color{negrew}$\mathit{-8.2\pm17}$ & \color{negrew}$\mathit{73\pm39}$ \\
P3O & \color{dom}$7.6\pm2.3$ & \color{dom}$11\pm17$ & \color{dom}$0.3\pm0.1$ & \color{dom}$15\pm11$ & \color{dom}$0.38\pm0.15$ & \color{dom}$28\pm12$ & \color{dom}$0.05\pm0.77$ & \color{dom}$36\pm18$ & \color{dom}$0.73\pm0.5$ & \color{dom}$40\pm17$ & \color{dom}$110.0\pm78$ & \color{dom}$0.028\pm0.025$ & $\boldsymbol{110\pm58}$ & $0.22\pm0.16$ \\
C-TRPO & \color{dom}$0.56\pm0.48$ & \color{dom}$3.9\pm7.4$ & \color{negrew}$-1.7\pm1.4$ & \color{negrew}$7.8\pm12.2$ & \color{negrew}$-1.3\pm0.5$ & \color{negrew}$8.9\pm11.5$ & \color{negrew}$-2.5\pm1.3$ & \color{negrew}$8.7\pm5.9$ & \color{negrew}$-2.1\pm0.6$ & \color{negrew}$8.5\pm6.7$ & \color{dom}$4.2\pm1.7$ & \color{dom}$0.0000078\pm0.000078$ & \color{negrew}$-18\pm8$ & \color{negrew}$0.058\pm0.18$ \\
\bottomrule
    \end{tabular}
    \label{tab:main-results}
    }
\end{table*}

The results are reported in \cref{tab:main-results}.
Our method consistently achieves the best performance in safe reward (among methods with at least 50\% safety probability) or safety probability (among methods with positive safe reward), demonstrating its ability to attain both high safety and meaningful task performance. When another method attains a higher safe reward, it typically comes at the expense of a substantial reduction in safety; conversely, methods with higher safety probabilities generally achieve significantly lower safe reward. 
 These trends also manifest for raw rewards and costs. \changed[dw]{Moreover, regardless of task performance, safety remains below near-perfect levels across all methods on the harder tasks.}

PPO-Lagrangian often collapses to poor reward, poor safety, or both. Baselines such as TRPO-Lagrangian, CPO, and C-TRPO can achieve higher safety on some harder tasks, but often at the cost of very low (mostly negative) rewards. CUP and FOCOPS generally exhibit lower safety.

\changed[dw]{In summary, our approach consistently achieves strong task performance and high levels of safety across challenging \emph{Safety Gymnasium} tasks.}

\subsubsection{Alignment of Policy Updates with Gradients}
\label{sec:align}
\begin{figure}
    \centering
    \hfill
    \begin{subfigure}{0.4\linewidth}
    \centering
    \includegraphics[width=0.9\linewidth]{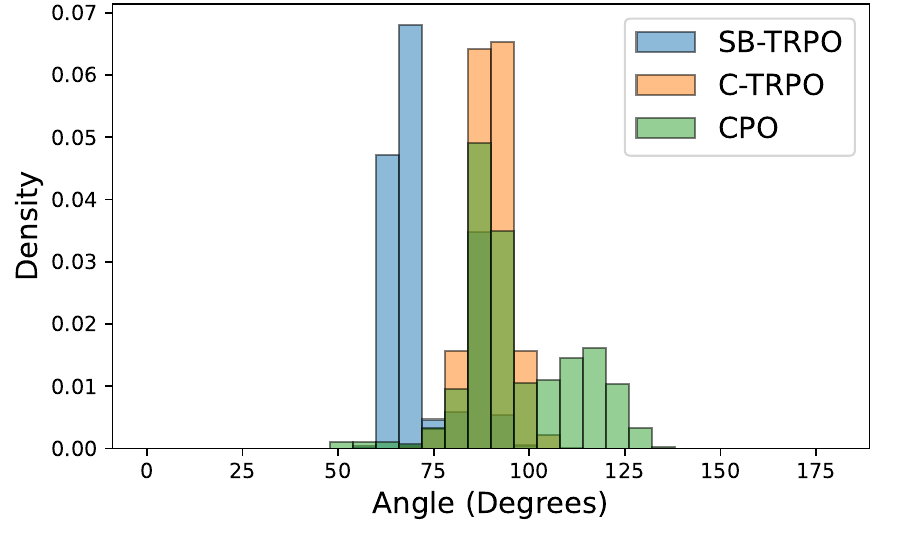}    
    \caption{Policy updates vs.\ reward gradients}
    \end{subfigure}
    \hfill
    \begin{subfigure}{0.4\linewidth}
    \centering
    \includegraphics[width=0.9\linewidth]{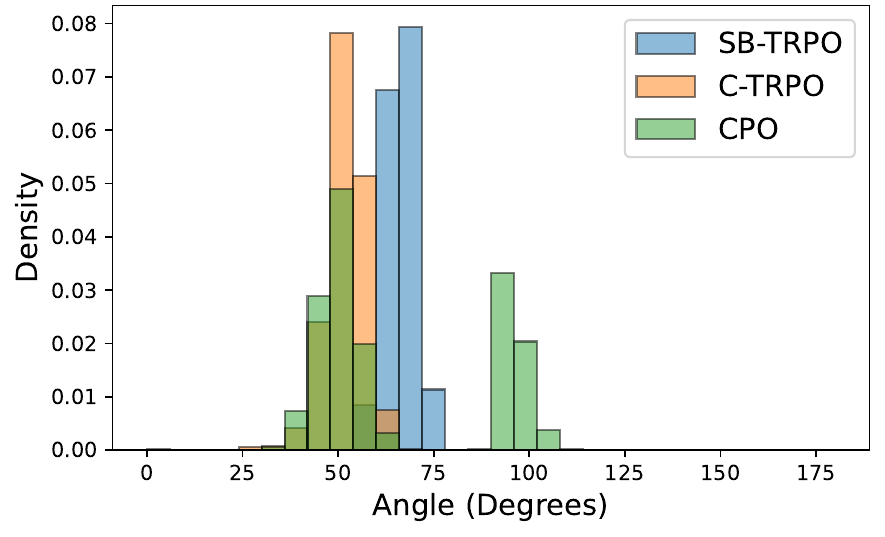}
    \caption{Policy updates vs.\ negated cost gradients}
    \end{subfigure}
    \hfill{}
    \caption{Angles between policy updates and reward/cost gradients for Car Circle}
    \label{fig:histr}    
\end{figure}

To further understand why SB-TRPO avoids the over-conservatism of similar methods, we analyse the angles between policy updates and reward/cost gradients (\cref{fig:histr,fig:histc}).  
CPO and C-TRPO updates are largely at best orthogonal to the reward gradient (angles around or above $90^\circ$), indicating that these methods spend considerable time in the recovery phase. In contrast, SB-TRPO updates lie below $90^\circ$, mostly around $60^\circ$–$70^\circ$ for Car Circle, confirming the performance improvement guarantee of \cref{thm:perf} empirically.  
At the same time, angles relative to the cost gradient can be larger for SB-TRPO than for CPO/C-TRPO but well below $75^\circ$, reflecting a measured approach to feasibility that balances safety with meaningful task progress.

\changed[dw]{
\subsubsection{Impact of Cost Limit on Baseline Safety}
\label{sec:clim}

We study the performance of baselines under a positive cost threshold of 15 (see \cref{tab:results-15}). This ad-hoc relaxation departs from the target zero-cost safety objective \pref{eq:relaxed} by \emph{explicitly allowing unsafe events}. On easier tasks such as the Velocity and Car Circle environments, the baselines exploit the higher cost allowance, becoming significantly less safe in line with the threshold, without substantially improving rewards. On harder tasks, rewards increase modestly at a clear cost to safety.

These results highlight the sensitivity to the choice of cost limit: treating it as a tunable, task-dependent hyperparameter can lead to significantly reduced safety without commensurate reward gains. 
}

\subsection{Ablation Studies}
\label{sec:ablation}

\paragraph{Safety Bias.} 
We evaluate the effect of the safety bias $\fimp$ over $\fimp \in [0.6, 0.9]$. 
\cref{fig:fimp_scatter} in \cref{app:fimp} shows that varying $\fimp$ shifts \alg policies 
along a nearly linear reward-safety probability Pareto frontier: 
higher values emphasise safety while lower values 
yield greater reward. 
The baselines typically underperform relative to this frontier (or achieve negative reward).
This demonstrates that \alg is not overly sensitive to the choice of $\fimp$, with $\fimp= 0.75$ achieving excellent results throughout (see \cref{tab:main-results}). 

\paragraph{Advantage Estimation.}

\changed[dw]{
 We ablate the effect of advantage estimation (see \cref{sec:setup}) in \cref{tab:critic} (\cref{app:critic}). At a fixed safety bias $\fimp=0.75$, the Monte Carlo (MC) variant of SB-TRPO attains both higher safety and lower reward than its GAE counterpart. Increasing the safety bias for GAE ($\fimp=0.8$) yields comparable safety and reward levels, indicating that the two variants achieve similar safety–reward trade-offs after a simple adjustment of the safety bias.
From a practical perspective, the MC variant obviates critic training, reducing computational cost by more than an order of magnitude (see \cref{tab:time}), motivating its use in our main experiments.

}

\section{Discussion}
\label{sec:discussion}
Our empirical results show that framing policy updates as \pref{eq:approx} --- dynamically balancing progress towards safety and task performance --- helps avoid collapse into trivially safe but task-ineffective regions under strict safety constraints. Unlike CPO and related methods such as C-TRPO, SB-TRPO has \emph{no separate recovery phase}, instead incorporating reward and cost objectives throughout training. Specifically, SB-TRPO's updates align well with both reward and cost gradients, whereas CPO-style updates are often at best orthogonal to the reward gradient. Consistent with \cref{thm:perf}, we observe \emph{steady improvement in both safety and task performance}.

Lagrangian methods face a different limitation in the hard-constraint regime: under a zero cost limit, Lagrange multipliers can only increase, making performance highly sensitive to initialisation. Small values can permit unsafe behaviour, while large values can induce persistent over-conservatism.

\changed[dw]{Finally, even methods which explicitly focus on maintaining feasibility such as CPO do not achieve near-perfect safety on the harder \emph{Safety Gymnasium} tasks. This suggests that the remaining empirical gap in safety is driven in part by factors shared across methods, particularly sparse cost signals and finite-sample gradient noise. In practice, we therefore recommend pairing SB-TRPO with deployment-time safeguards.}

\paragraph{Conclusion.}
In summary, SB-TRPO is a principled approach to hard-constrained RL that achieves strong task performance alongside high levels of safety on \emph{Safety Gymnasium} tasks.

\paragraph{Limitations.}
As is common \citep{TRPO,guetalaaai2024}, \cref{thm:perf} assumes exact gradients, while \cref{thm:optideal} assumes finite MDPs. \alg does not directly support arbitrary positive cost limits, although \cref{thm:optideal} guarantees the lowest achievable cost, even when positive. 

\subsection*{AI use statement}

In this work, we used generative AI tools to provide feedback on research methodology and assist with checking lemmas and proofs. We have not used generative AI tools to generate synthetic datasets, propose hypotheses, implement methods, assist with translation, clean or reformat datasets, support qualitative or thematic data analysis, or interpret experimental results, and the remaining required-disclosure tasks are not applicable to this work. Additionally, we used generative AI tools to review related work and improve the clarity and presentation of the manuscript. We have reviewed all AI-assisted work and independently verified mathematical arguments. We take responsibility for the final content of this work, including text, claims, and artifacts produced with the aid of generative AI.

\subsection*{Ethics statement}

This work does not involve human subjects or personal data. All experiments are conducted in simulation using publicly available \emph{Safety Gymnasium} environments. We are not aware of any additional ethical concerns arising from the methods or experiments presented in this work.

\subsection*{Reproducibility statement}

Full details of the experimental setup are provided in \cref{sec:setup,app:setup}, including instructions for reproducing our results in \cref{app:rep}, whilst code is provided in the supplementary materials. Complete proofs of our theoretical results are provided in \cref{app:method}.

\bibliography{references}
\bibliographystyle{iclr2027_conference}

\appendix

\newpage

\section{Motivation of Zero-Cost CMDP Formulation}
In many safety-critical domains, the appropriate requirement is that the agent \emph{never} enters unsafe states (e.g.\ when they encode damage to an expensive robot or even accidents of autonomous cars).  In the discounted-cost CMDP formalism with non-negative costs, this requirement is expressed \emph{exactly} by imposing the zero-threshold constraint \(J_c(\pi)=0\). Introducing a positive threshold \(b>0\) effectively permits some violations and, crucially, adds a hyperparameter that encodes an arbitrary notion of ``acceptable risk’’. 

The level of violation required for learning to progress varies greatly across tasks (e.g.\ around 17 on the hardest navigation environments, around 10 for PointPush, and well below~1 for velocity or circle tasks) as well as across algorithms, neural network architectures and optimisation hyperparameters. This makes positive-threshold CMDPs brittle, environment-dependent, and potentially misaligned with true safety objectives.
Conceptually, a positive cost threshold in the context of hard safety is undesirable because it conflates the \emph{problem statement/specification} with an \emph{algorithmic hyper-parameter}.

For these reasons, we focus on the \emph{intrinsically meaningful} zero-threshold (hard-constraint) CMDP, which avoids the ambiguity of selecting \(b\), removes the tuning burden associated with calibrating allowable violations, and directly matches the requirement of ``almost surely no unsafe events’’. 

Finally, we believe that the zero-cost problem is significantly under-explored in model-free Safe RL, despite being the more appropriate formulation for capturing critical safety violations in practice. By explicitly targeting this regime, our work aims to help shift attention towards this practically important but comparatively understudied problem class.

\section{Supplementary Materials for \cref{sec:method}}
\label{app:method}
\subsection{Supplementary Materials for \cref{sec:ideal}}
\label{app:ideal}

\dw{double-check full support of $\pi_k$ not needed}

\dw{NB all optima attained, ties broken arbitrarily}

\subsubsection{Details for \cref{ex:cpofail}}
\label{app:cpofail}

Recall the MDP in \cref{fig:cpofail}.
Suppose $\pi_0$ is such that $\pi_0(a\mid s_0),\pi_0(a\mid s_1)\in(0,1)$ and let $\pi_k$ for $k\geq 1$ be obtained by iterating the \pref{eq:recov} or \pref{eq:cpo} updates (whichever applicable).

First, note that due to the KL constraint for all $k\in\nat$, $\pi_k(a\mid s_0),\pi_k(a\mid s_1)\in(0,1)$ (otherwise the KL is infinite). 

Moreover, for all $k\in\nat$, $\pi_{k+1}(a\mid s_0)\leq\pi_k(a\mid s_0)$. Otherwise, there exists $k\in\nat$ with $\pi_{k+1}(a\mid s_0)>\pi_k(a\mid s_0)$. Define $\pi'_{k+1}$ by $\pi'_{k+1}(a\mid s_0)\defeq\pi_k(a\mid s_0)$ and $\pi'_{k+1}(a\mid s_1)\defeq\pi_{k+1}(a\mid s_1)$. Note that $\dklo{\pi_k}{\pi'_{k+1}}<\dklo{\pi_k}{\pi_{k+1}}$ and $J_c(\pi'_{k+1})<J_c(\pi_{k+1})$ since $\pi_{k+1}(a\mid s_1)>0$.

\subsubsection{Proof of \cref{thm:optideal}}
\label{app:optideal}

In this subsection we assume the MDP is finite. For a policy $\pi$, recall the discounted state and state-action occupancy measures
\begin{align*}
    d_\pi(s)&\defeq(1-\gamma)\cdot\sum_{t\in\nat}\gamma^t\cdot\mathbb E_{\tau\sim\pi}\left[s_t=s\right]\\
    d_\pi(s,a)&\defeq(1-\gamma)\cdot\sum_{t\in\nat}\gamma^t\cdot\mathbb E_{\tau\sim\pi}\left[s_t=s,a_t=a\right]=\pi(a\mid s)\cdot d_\pi(s)
\end{align*}
Besides, for policies $\pi_0$ and $\pi_1$ and $\lambda\in[0,1]$ define the mixture $(1-\lambda)\cdot\pi_0+\lambda\cdot\pi_1$ by
\begin{align*}
((1-\lambda)\cdot\pi_0+\lambda\cdot\pi_1)(a\mid s)&\defeq
    \begin{cases}
            \frac{(1-\lambda)\cdot d_{\pi_0}(s,a)+\lambda\cdot d_{\pi_1}(a,s)}{(1-\lambda)\cdot d_{\pi_0}(s)+\lambda\cdot d_{\pi_1}(s)}&\text{if } (1-\lambda)\cdot d_{\pi_0}(s)+\lambda\cdot d_{\pi_1}(s)\neq 0\\
            \pi_0(a\mid s) &\text{otherwise}
        \end{cases}
\end{align*}
 Note that $d_{(1-\lambda)\cdot\pi_0+\lambda\cdot\pi_1}=(1-\lambda)\cdot d_{\pi_0}+\lambda\cdot d_{\pi_1}$ and hence $J_r((1-\lambda)\cdot\pi_0+\lambda\cdot\pi_1)=(1-\lambda)\cdot J_r(\pi_0)+\lambda\cdot J_r(\pi_1)$ (similarly for cost).

\begin{lemma}
\label{lem:mixKL3}
    Let $\pi$ and $\pi'$ be policies and $\lambda\in[0,1]$. Then $\dklo\pi{(1-\lambda)\cdot\pi+\lambda\cdot\pi'}\leq\dklo\pi{\pi'}$.
\end{lemma}
\begin{proof}
Define $\pi''\defeq(1-\lambda)\cdot\pi+\lambda\cdot \pi'$.
    Let $s\in\mathcal S$ be such that $d_\pi(s)>0$. Note that for $\alpha(s)\defeq\frac{\lambda\cdot d_{\pi'}(s)}{(1-\lambda)\cdot d_\pi(s)+\lambda\cdot d_{\pi'}(s)}\in[0,1]$, $\pi''(a\mid s)=(1-\alpha(s))\cdot\pi(a\mid s)+\alpha(s)\cdot\pi'(a\mid s)$. Hence, by convexity of the KL-divergence,
    \begin{align*}
        \dklo{\pi(\cdot\mid s)}{\pi''(\cdot\mid s)}&\leq (1-\alpha(s))\cdot\dklo{\pi(\cdot\mid s)}{\pi(\cdot\mid s)}+\alpha(s)\cdot\dklo{\pi(\cdot\mid s)}{\pi'(\cdot\mid s)}\\
        &\leq\dklo{\pi(\cdot\mid s)}{\pi'(\cdot\mid s)}
    \end{align*}
    and the claim follows since $d_\pi(s)>0$ was arbitrary.
\end{proof}

\begin{lemma}
\label{lem:mixKL}
    Let $\pi,\pi'$ be policies. Then
    \begin{align*}
        \lim_{\lambda\searrow 0}\dklo\pi{(1-\lambda)\cdot\pi +\lambda\cdot\pi'}=0
    \end{align*}
\end{lemma}
\begin{proof}
    Let $s\in\mathcal S$ and $a\in\mathcal A$ be such that $d_\pi(s)>0$ and $\pi(a\mid s)>0$. Hence, for $\lambda\in[0,1)$,
    \begin{align*}
        (1-\lambda)\cdot d_\pi(s,a)+\lambda\cdot d_{\pi'}(a,s)>0\\
        (1-\lambda)\cdot d_\pi(s)+\lambda\cdot d_{\pi'}(s)>0
    \end{align*}
    and by continuity of $\log$ on $(0,\infty)$, 
    \begin{align*}
        \lim_{\lambda\searrow 0}\log\frac{(1-\lambda)\cdot d_\pi(s,a)+\lambda\cdot d_{\pi'}(a,s)}{(1-\lambda)\cdot d_\pi(s)+\lambda\cdot d_{\pi'}(s)}=\log \frac{d_\pi(s,a)}{d_\pi(s)} =\log\pi(a\mid s)
    \end{align*}
    Since $\mathcal S$ and $\mathcal A$ are finite, the result follows.
\end{proof}

\begin{lemma}
\label{lem:mixKL2}
    Let $\pi'$ be a policy and  $(\pi_k)_{k\in\nat}$ be a sequence of policies converging to $\bar\pi$, i.e.\ $\lim_{k\to\infty} \pi_k(a\mid s)=\bar\pi(a\mid s)$ for all $s\in\mathcal S$ and $a\in\mathcal A$. Then for $\lambda\in(0,1)$,
    \begin{align*}
        \lim_{k\to\infty} \dklo{\pi_k}{(1-\lambda)\cdot\pi_k+\lambda\cdot\pi'}=\dklo{\bar\pi}{(1-\lambda)\cdot\bar\pi+\lambda\cdot\pi'}
    \end{align*}
\end{lemma}
\begin{proof}
To reduce notational clutter we use $d_k$ for $d_{\pi_k}$, $\bar d$ for $d_{\bar\pi}$, and $d'$ for $d_{\pi'}$.
    Note that
    \begin{align*}
        \lim_{k\to\infty} \pi_k(a\mid s)\cdot\log\pi_k(a\mid s)&= \bar\pi(a\mid s)\cdot\log\bar\pi(a\mid s)
    \end{align*}
    \dw{also occupancy meaure converges}
    Hence, it suffices to show
    \begin{align*}
        &\lim_{k\to\infty} d_k(s)\cdot\pi_k(a\mid s)\cdot\log\frac{(1-\lambda)\cdot d_k(s)+\lambda\cdot d'(s)}{(1-\lambda)\cdot d_k(s,a)+\lambda\cdot d'(s,a)}\\
        &= \bar d(s)\cdot\bar\pi(a\mid s)\cdot\log\frac{(1-\lambda)\cdot\bar d(s)+\lambda\cdot d'(s)}{(1-\lambda)\cdot\bar d(s,a)+\lambda\cdot d'(s,a)}
    \end{align*}
    for all $s\in\mathcal S$ and $a\in\mathcal A$.
    
    First if $d'(s)=0$, hence, $d'(s,a)=0$, then
        \begin{align*}
        &\lim_{k\to\infty} d_k(s)\cdot\pi_k(a\mid s)\cdot\log\frac{(1-\lambda)\cdot d_k(s)+\lambda\cdot d'(s)}{(1-\lambda)\cdot d_k(s,a)+\lambda\cdot d'(s,a)}\\
        &= - \lim_{k\to\infty} d_k(s)\cdot\pi_k(a\mid s)\cdot\log\pi_k(s,a)\\
        &= - \bar d(s)\cdot\bar\pi(a\mid s)\cdot\log\bar\pi(s,a)\\
        &= \bar d(s)\cdot\bar\pi(a\mid s)\cdot\log\frac{(1-\lambda)\cdot\bar d(s)+\lambda\cdot d'(s)}{(1-\lambda)\cdot\bar d(s,a)+\lambda\cdot d'(s,a)}
    \end{align*}
    \dw{should deal with the special case for 0 denominator explicitly}

    Next, if $d'(s)>0$ 
    then 
    \begin{align*}
        &\lim_{k\to\infty} \log\left((1-\lambda)\cdot d_k(s)+\lambda\cdot d'(s)\right)= \log(1-\lambda)\cdot\bar d(s)+\lambda\cdot d'(s)
    \end{align*}
    due to $\lambda\cdot d'(s)>0$ and the continuity of $\log$ on $(0,\infty)$.
    Hence, it suffices to show
    \begin{align*}
        &\lim_{k\to\infty} d_k(s)\cdot\pi_k(a\mid s)\cdot\log\left((1-\lambda)\cdot d_k(s,a)+\lambda\cdot d'(s,a)\right)\\
        &= \bar d(s)\cdot\bar\pi(a\mid s)\cdot\log\left((1-\lambda)\cdot\bar d(s,a)+\lambda\cdot d'(s,a)\right)
    \end{align*}
    If $d'(s,a)>0$ then this is again due to the continuity of $\log$. Otherwise
    \begin{align*}
    &\lim_{k\to\infty} d_k(s)\cdot\pi_k(a\mid s)\cdot\log\left((1-\lambda)\cdot d_k(s,a)+\lambda\cdot d'(s,a)\right)\\
    &=\lim_{k\to\infty} d_k(s,a)\cdot\log(1-\lambda)+\lim_{k\to\infty} d_k(s,a)\cdot\log d_k(s,a)\\
    &=\bar d(s,a)\cdot\log(1-\lambda)+\bar d(s)\cdot\bar\pi(a\mid s)\cdot\log\bar d(s,a)\\
    &= \bar d(s)\cdot\bar\pi(a\mid s)\cdot\log\left((1-\lambda)\cdot\bar d(s,a)+\lambda\cdot d'(s,a)\right)
\end{align*}
since $\lambda<1$.
\end{proof}

\begin{lemma}
\label{lem:costopt}
    Let $\fimp\in (0,1]$, $\pi_0$ be a policy and $\pi_k$ be obtained via \pref{eq:general} for $k\geq 1$. Then $\lim_{k\to\infty} J_c(\pi_k)=\min_\pi J_c(\pi)$ (in particular $J_c(\pi_k)$ converges).
\end{lemma}
\begin{proof}
    First, note that since $\dklo{\pi}{\pi}=0$ for all policies $\pi$, $J_c(\pi_{k+1})\leq J_c(\pi_k)$ for all $k\in\nat$. Hence, since $J_c$ is bounded, $J_c(\pi_k)$ converges.
    
    We assume towards contradiction that there exists a policy $\pi'$ such that $J_c(\pi')<\lim_{k\to\infty} J_c(\pi_k)$. Clearly, there  exists a subsequence $(\pi_{k_i})_{i\in\nat}$ of $(\pi_k)_{k\in\nat}$ for which $\pi_{k_i}$ converges to some policy $\bar\pi$, i.e.\  $\lim_{i\to\infty}\pi_{k_i}(a\mid s)=\bar\pi(a\mid s)$ for all $s\in\mathcal S$ and $a\in\mathcal A$. By \cref{lem:mixKL3,lem:mixKL} there exists a policy $\pi''$ satisfying
    $\dklo{\bar\pi}{\frac 1 2\bar\pi+\frac 1 2\pi''}<\delta$ and $J_c(\pi'')< J_c(\bar\pi)$. By \cref{lem:mixKL2} there exists $K\in\nat$ such that $\dklo{\pi_k}{\frac 1 2\pi_k+\frac 1 2\pi''}<\delta$ for all $k\geq K$. Moreover, there exists $k\geq K$ such that $(1-\frac\fimp 2)\cdot J_c(\pi_k)+\fimp\cdot J_c(\pi'')< J_c(\bar\pi)$. Hence, $J_c(\pi_{k+1})\leq J_c(\pi_k)-\fimp(J_c(\pi_k)-(\frac 1 2 J_c(\pi_k)+\frac 1 2 J_c(\pi'')))<J_c(\bar\pi)$ and due to monotonicity, $\lim_{k\to\infty} J_c(\pi_k)<J_c(\bar\pi)$, which is a contradiction since $J_c$ is continuous.
\end{proof}

\optideal*
\begin{proof}
    Let $\pi^*$ be a policy satisfying $J_c(\pi^*)=J_c^*$ and $J_r(\pi^*)=J_r^*$.
    The first part is due to \cref{lem:costopt}.
    
    To show the second, note that because of $\lim_{k\to\infty}J_c(\pi_k)=J^*_c$ and continuity, $\limsup_{k\to\infty} J_r(\pi_k)\leq J_r^*$. Hence, it suffices to show $\liminf_{k\to\infty} J_r(\pi_k)\geq J_r^*$.

    Otherwise, there exists a subsequence $\pi_{j_i}$ for which $J_r(\pi_{j_i})$ converges to $\lim_{i\to\infty} J_r(\pi_{j_i})=\liminf_{k\to\infty} J_r(\pi_k)<J_r^*$.

    Due to compactness of $[0,1]^{\mathcal S\times\mathcal A}$ there exists a subsequence $k_i$ of $j_i$ for which $\pi_{k_i-1}$ converges to some $\bar\pi$.

    By \cref{lem:mixKL} there exists $\lambda\in (0,1]$ such that for $\pi'\defeq (1-\lambda)\cdot\bar\pi+\lambda\cdot\pi^*$ it holds
     $\dklo{\bar\pi}{\pi'}<\delta$.
     Note that $J_c(\pi')=c^*=J_c(\pi^*)=J_c(\bar\pi)$. Besides, by \cref{lem:mixKL3}, $\dklo{\bar\pi}{(1-\fimp)\cdot\bar\pi+\fimp\cdot\pi'}<\delta$.
     
    Now for $i\in\nat$ define $\pi'_{k_i}\defeq (1-\fimp)\cdot\pi_{k_i-1}+\fimp\cdot\pi'$.
    Since $\fimp<1$, by \cref{lem:mixKL2} there exists $I\in\nat$ such that for all $i\geq I$,
    $\dklo{\pi_{k_i-1}}{\pi'_{k_i}}<\delta$. Hence, due to $J_c(\pi'_{k_i})=(1-\fimp)\cdot J_c(\pi_{k_i-1})+\fimp\cdot J_c^*$, the policy $\pi'_{k_i}$ is feasible for the update and
    \begin{align*}
        \liminf_{k\to\infty} J_r(\pi_k)&=\lim_{i\to\infty} J_r(\pi_{k_i})\geq \liminf_{i\to\infty} J_r(\pi'_{k_i})\\
        &= (1-\fimp)\cdot \lim_{i\to\infty} J_r(\pi_{k_i-1})+\fimp\cdot J_r(\pi')=(1-\fimp)\cdot J_r(\bar\pi)+\fimp\cdot J_r(\pi')
    \end{align*}
    This is a contradiction since $J_r(\bar\pi)\geq\liminf_{k\to\infty} J_r(\pi_k)$ and 
    $J_r(\pi')>\liminf_{k\to\infty} J_r(\pi_k)$.
\end{proof}

\subsection{Supplementary Materials for \cref{sec:approx}}
\label{app:approxi}

\begin{lemma}
\label{lem:up}
    $\mathcal L_{r,\po}(\pi) 
    - C_{r,\po}\cdot D_{\mathrm{KL}}^{\max}(\po\parallel\pi)\leq J_r(\pi)\leq \mathcal L_{r,\po}(\pi) 
    + C_{r,\po}\cdot D_{\mathrm{KL}}^{\max}(\po\parallel\pi)$.
\end{lemma}
\begin{proof}
     By \citep[Thm.~1]{TRPO},
    \begin{align*}
        \mathcal L_{r,\po}(\pi) 
    - C_{r,\po}\cdot D_{\mathrm{KL}}^{\max}(\po\parallel\pi)&\leq J_r(\pi),&
    C_{r,\po}&\defeq\frac{4\cdot\gamma\cdot\max_{s,a}|A_r^{\po}(s,a)|}{(1-\gamma)^2}
    \end{align*}
    Consider the reward function $-r$. Note that $J_{-r}(\pi)=-J_{r}(\pi)$ and $A_{-r}^{\po}(s,a)=-A_r^{\po}(s,a)$. Thus, $C_{-r,\po}=C_{r,\po}$ and $\mathcal L_{-r,\po}(\pi)=-\mathcal L_{r,\po}(\pi)$. Finally, again by \citep[Thm.~1]{TRPO},
    \begin{align*}
        -J_r(\pi)=J_{-r}(\pi)&\geq\mathcal L_{-r,\po}(\pi) 
    - C_{-r,\po}\cdot D_{\mathrm{KL}}^{\max}(\po\parallel\pi)\\
    &=-\mathcal L_{r,\po}(\pi) 
    - C_{r,\po}\cdot D_{\mathrm{KL}}^{\max}(\po\parallel\pi)&&\qedhere
    \end{align*}
\end{proof}

We recall the Fisher information matrix:
\begin{align*}
F \defeq \mathbb{E}_{s,a \sim \pi_{\parao}} 
\left[ \nabla_\theta \log \pi_\theta(a|s)\big|_{\theta=\parao}\cdot \nabla_\theta \log \pi_\theta(a|s)^\top\big|_{\theta=\parao} \right]
\end{align*}

\begin{lemma}
\label{lem:kl}
For $\del_r$ and $\del_c$ defined in \cref{eq:approxr,eq:approxc},
\begin{align*}
\frac 1 2\cdot\del^T\cdot F\cdot\del\leq\epskl
\end{align*}
where $\del= (1-\conv)\cdot \del_r + \conv\cdot \del_c$ for arbitrary $\conv\in[0,1]$
\end{lemma}
\begin{proof}
    Since $F$ is positive semi-definite, $\del\mapsto\del^T\cdot F\cdot\del$ is convex and the bound follows since it is satisfied by both $\del_r$ and $\del_c$.
\end{proof}

\begin{lemma}
\label{lem:conv}
    Assume $\dotp{g_c}{\del_c}\leq-\epsc$ and define $\del_\mu\defeq (1-\conv)\cdot \del_r + \conv\cdot \del_c$.
    Then
    \begin{align*}
        \max_{\conv} \dotp{g_r}{\del_\conv}\qquad\text{s.t.}\quad\dotp{g_c}{\del_\conv}\leq-\epsc,\quad\conv\in[0,1]
    \end{align*}
    is feasible and its optimal solution is $\conv$ defined in \cref{eq:conv}.
\end{lemma}
NB for our choice $\epsc\defeq-\fimp\cdot\dotp{g_c}{\del_c}\geq 0$, where $\fimp\in(0,1]$, the assumption clearly holds.
\begin{proof}
By definition of $\del_r$ and $\del_c$,
\begin{align}
\label{eq:gdel}
    \dotp{g_r}{\del_r}&\geq\dotp{g_r}{\del_c}&
        \dotp{g_c}{\del_c}&\leq\dotp{g_c}{\del_r}
\end{align}

    Therefore, we can equivalently optimise
    \begin{align*}
        \min_{\conv} \conv\qquad\text{s.t.}\quad\dotp{g_c}{\del_\conv}\leq-\epsc,\quad\conv\in[0,1]
    \end{align*}
    If $\dotp{g_c}{\del_r}=\dotp{g_c}{\del_c}$ or $\dotp{g_c}{\del_r}<-\epsc$ then the optimal solution is clearly $\conv\defeq 0$.

    Otherwise, first note that by the premise of this lemma and \cref{eq:gdel}, the ratio in \cref{eq:conv} is always in $[0,1]$. Besides, it is straightforward to see that $\dotp{g_c}{\del_\conv}\leq-\epsc$ iff
    \begin{align*}
        \conv\geq    \frac{\epsc + \dotp{g_c} {\del_r}}{\dotp{g_c}{\del_r}-\dotp{g_c}{\del_c}}&\qedhere
    \end{align*}
\end{proof}

\begin{lemma}[Exact solution of \pref{eq:approx}]
\label{prop:analytic}
Let $g_r,g_c\in\mathbb R^d$, $F\in\mathbb R^{d\times d}$ and let $\epsc>0$, $\epskl>0$. 
We assume the \emph{non-degenerate case}:
\[
g_r \neq 0, \quad g_c \neq 0, \quad g_c \not\parallel g_r, \quad F \succ 0.
\]

Then the optimal solution $\del^*$ of \pref{eq:approx} is
\[
\del^* =
\begin{cases}
\sqrt{\frac{2\epskl}{a}}\, F^{-1} g_r, & \text{if } \sqrt{\frac{2\epskl}{a}}\, b \le -\epsc,\\[0.5em]
\sqrt{\frac{2\epskl}{a - 2 \lambda^* b + (\lambda^*)^2 c}}\, F^{-1} (g_r - \lambda^* g_c), & \text{otherwise,}
\end{cases}
\]
where 
\[
a := \dotp{g_r}{F^{-1} g_r},\quad 
b := \dotp{g_r}{F^{-1} g_c},\quad 
c := \dotp{g_c}{F^{-1} g_c}.
\]
and $\lambda^*>0$ is the unique positive solution of the quadratic
\[
A \lambda^2 + B \lambda + C = 0, 
\quad\text{with}\quad
A = 2 \epskl c^2 - \epsc^2 c,\quad
B = -4 \epskl b c + 2 \epsc^2 b,\quad
C = 2 \epskl b^2 - \epsc^2 a.
\]
\end{lemma}

\begin{proof}
We form the Lagrangian with multipliers $\lambda \ge 0$ (linear constraint) and $\nu \ge 0$ (quadratic constraint):
\[
\mathcal{L}(\del, \lambda, \nu) = \dotp{g_r}{\del} + \lambda(-\epsc - \dotp{g_c}{\del}) + \nu\left(\epskl - \tfrac12 \del^\top F \del\right).
\]

Stationarity gives
\[
g_r - \lambda g_c - \nu F \del = 0
\quad \Longrightarrow \quad
\del = \frac{1}{\nu} F^{-1} (g_r - \lambda g_c), \quad \nu>0.
\]

\paragraph{Justification that $\nu>0$.} 
The objective is linear in $\del$ and $g_r \neq 0$, so it is unbounded in the direction of $g_r$ without the quadratic constraint.  
The linear constraint $\dotp{g_c}{\del} \le -\epsc$ defines a half-space. In the non-degenerate case $g_c \not\parallel g_r$, this half-space alone cannot bound the linear objective. Therefore, the quadratic constraint must be active at the solution, which by complementary slackness implies $\nu>0$.

\paragraph{Form of KKT candidates.}
Substituting $\del$ into the quadratic constraint $\tfrac12 \del^\top F \del = \epskl$ gives
\[
\nu = \sqrt{\frac{(g_r - \lambda g_c)^\top F^{-1} (g_r - \lambda g_c)}{2 \epskl}}.
\]
Hence all KKT candidates satisfy
\begin{align*}
    \del(\lambda) &= \sqrt{\frac{2 \epskl}{(g_r - \lambda g_c)^\top F^{-1} (g_r - \lambda g_c)}}\cdot F^{-1}\cdot (g_r - \lambda g_c)\\
    &=\sqrt{\frac{2 \epskl}{a - 2 \lambda b + \lambda^2 c}}\cdot F^{-1}\cdot (g_r - \lambda g_c)
\end{align*}
using the above definitions of $a$, $b$ and $c$, so that
\begin{align*}
    \dotp{g_c}{\del(\lambda)} = \sqrt{\frac{2 \epskl}{a - 2 \lambda b + \lambda^2 c}}\cdot (b - \lambda c)
\end{align*}

\paragraph{Complementary slackness for the linear constraint.}
Either $\lambda=0$ (linear constraint inactive) or $\lambda>0$ with equality $\dotp{g_c}{\del(\lambda)} = -\epsc$.

\paragraph{Case $\lambda=0$.} If $\dotp{g_c}{\del(0)}= \sqrt{\frac{2\epskl}{a}}\, b\le -\epsc$, then $\del^* = \del(0)$.

\paragraph{Case $\lambda>0$.} Solving $\dotp{g_c}{\del(\lambda)} = -\epsc$ and squaring both sides gives the quadratic equation $A \lambda^2 + B \lambda + C=0$ displayed in the proposition. The unique positive root $\lambda^*$ yields the optimal step $\del^* = \del(\lambda^*)$.

\paragraph{Summary.} In the non-degenerate case, the quadratic constraint is active ($\nu>0$) to ensure a finite optimum. The linear constraint determines $\lambda^*$, and $\del^*$ follows directly. 
\end{proof}

\subsubsection{Proof of \cref{thm:perf}}
\label{app:perf}

For $M>0$, let $B_M(\theta)$ be the $M$-ball around $\theta$.
\begin{lemma}
\label{lem:perf}
    Assume $J_r$ and $J_c$ are $L$-Lipschitz smooth on $B_M(\theta)$ and $\|\del_r\|,\|\del_c\|\leq M$ for some $M>0$. Then for all $\scale\in[0,1]$,
\begin{enumerate}[noitemsep,topsep=0pt]
    \item $J_c(\theta+\scale\cdot\del)\leq J_c(\theta)+\scale\cdot\fimp\cdot\dotp {g_c}{\del_c}+\frac {L\cdot\scale^2}2\cdot\|\del\|^2$
    \item $J_r(\theta+\scale\cdot\del)\geq J_r(\theta)+\scale\cdot(1-\conv)\cdot\dotp{g_r}{\del_r}+\scale\cdot\conv\cdot\dotp {g_r}{\del_c}-\frac {L\cdot\scale^2}2\cdot\|\del\|^2$
    \item if $g_c= 0$ then 
        $J_r(\theta+\scale\cdot\del)\geq J_r(\theta)+\scale\cdot\dotp {g_r}{\del_r}-\frac {L\cdot\scale^2}2\cdot\|\del\|^2$.
\end{enumerate} 
\end{lemma}
\begin{proof}
First, note that $g_r=\nabla J_r(\theta)$ and $g_c=\nabla J_c(\theta)$. 
    Therefore, by assumption and Taylor's theorem, for every $\scale\in[0,1]$,
    \begin{align*}
        J_c(\theta+\scale\cdot\del)\leq J_c(\theta)+\scale\cdot\dotp{g_c}\del+\frac {L\cdot\scale^2} 2\cdot\|\del\|^2\\
        J_r(\theta+\scale\cdot\del)\geq J_r(\theta)+\scale\cdot\dotp{g_r}\del-\frac {L\cdot\scale^2} 2\cdot\|\del\|^2
    \end{align*}
    since $\|\scale\cdot\del\|\leq M$.
    By the choice of $\epsc$ and \cref{lem:conv}, $\dotp{g_c}\del\leq\fimp\cdot\dotp{g_c}{\del_c}$ and the first claim follows.
    The second also follows directly by definition of $\del$ as a convex combination. Finally, if $g_c=0$ then $\conv=0$ and the third claim follows from the second.
\end{proof}

\perf*
\begin{proof}
Since $J_r$ and $J_c$ are smooth, they are $L$-smooth on the bounded set $B_M(\theta)$ for $M\defeq\max\{\|\del_r\|,\|\del_c\|\}$ and sufficiently large $L>0$.

    If $g_c\neq 0$ then $\dotp{g_c}{\del_c}<0$ and the claim follows from \cref{lem:perf} for sufficiently small $\scale\in(0,1]$.

    Next, suppose that $g_r\neq 0$ and $\dotp{g_r}{\del_c}\geq 0$. Note that $\mu<1$ if $\fimp<1$. By definition of $\del_r$ in \cref{eq:approxr}, $\dotp{g_r}{\del_r}> 0$.
    Besides, by \cref{lem:perf}, 
    \begin{align*}
        J_r(\theta+\scale\cdot\del)\geq J_r(\theta)+\scale\cdot(1-\conv)\cdot\dotp{g_r}{\del_r}-\frac {L\cdot\scale^2}2\cdot\|\del\|^2
    \end{align*}
and the claim follows for sufficiently small $\scale\in(0,1]$.
\end{proof}

\section{Supplementary Materials for \cref{sec:exp}}
\subsection{Supplementary Materials for \cref{sec:setup}}
\label{app:setup}
\subsubsection{Safety Gymnasium Tasks}
\label{app:sge}

We provide additional details on the Safety Gymnasium tasks used in our experiments (see \citep{ji2023safety} and \url{https://safety-gymnasium.readthedocs.io/} for the full details). 

Safety Gymnasium uses reward shaping to provide the agent with auxiliary guidance signals that facilitate more efficient policy optimisation. 
In the Goal tasks of Safe Navigation environments, for example, the agent receives dense rewards based on its movement relative to the goal, supplementing the sparse task completion reward, so that the cumulative reward can be negative if the agent moves significantly away from the goal. 
A cost is also computed to quantify constraint violations, with scales that are task dependent. 
In our selected tasks, these violations include contact with hazards, displacement of hazards, and leaving safe regions in Safe Navigation, and exceeding velocity limits in Safe Velocity. 

\begin{figure}
    \centering
    \hfill
    \begin{subfigure}{0.24\linewidth}
        \includegraphics[width=\linewidth]{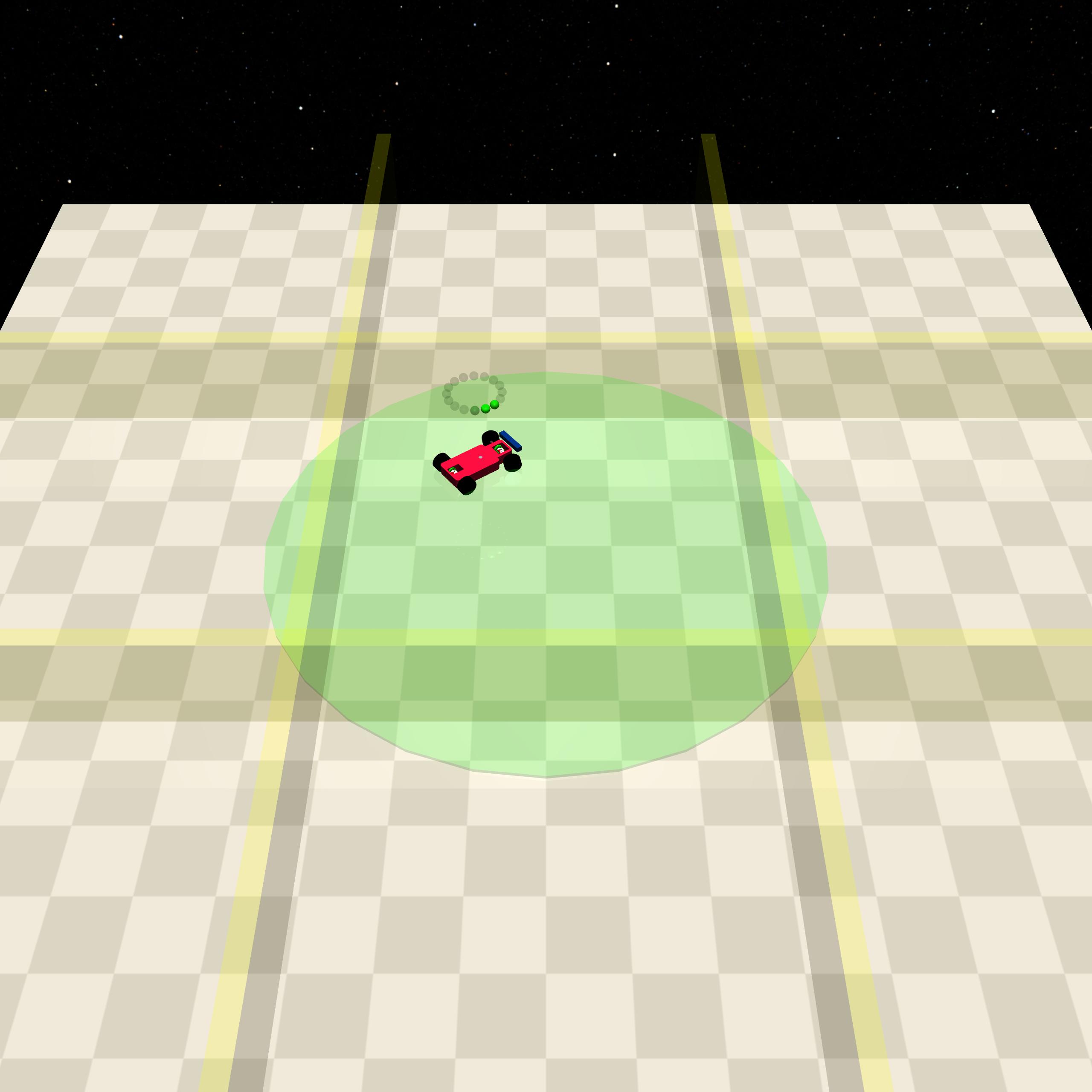}
        \caption{Circle}
    \end{subfigure}
    \hfill
    \begin{subfigure}{0.24\linewidth}
        \includegraphics[width=\linewidth]{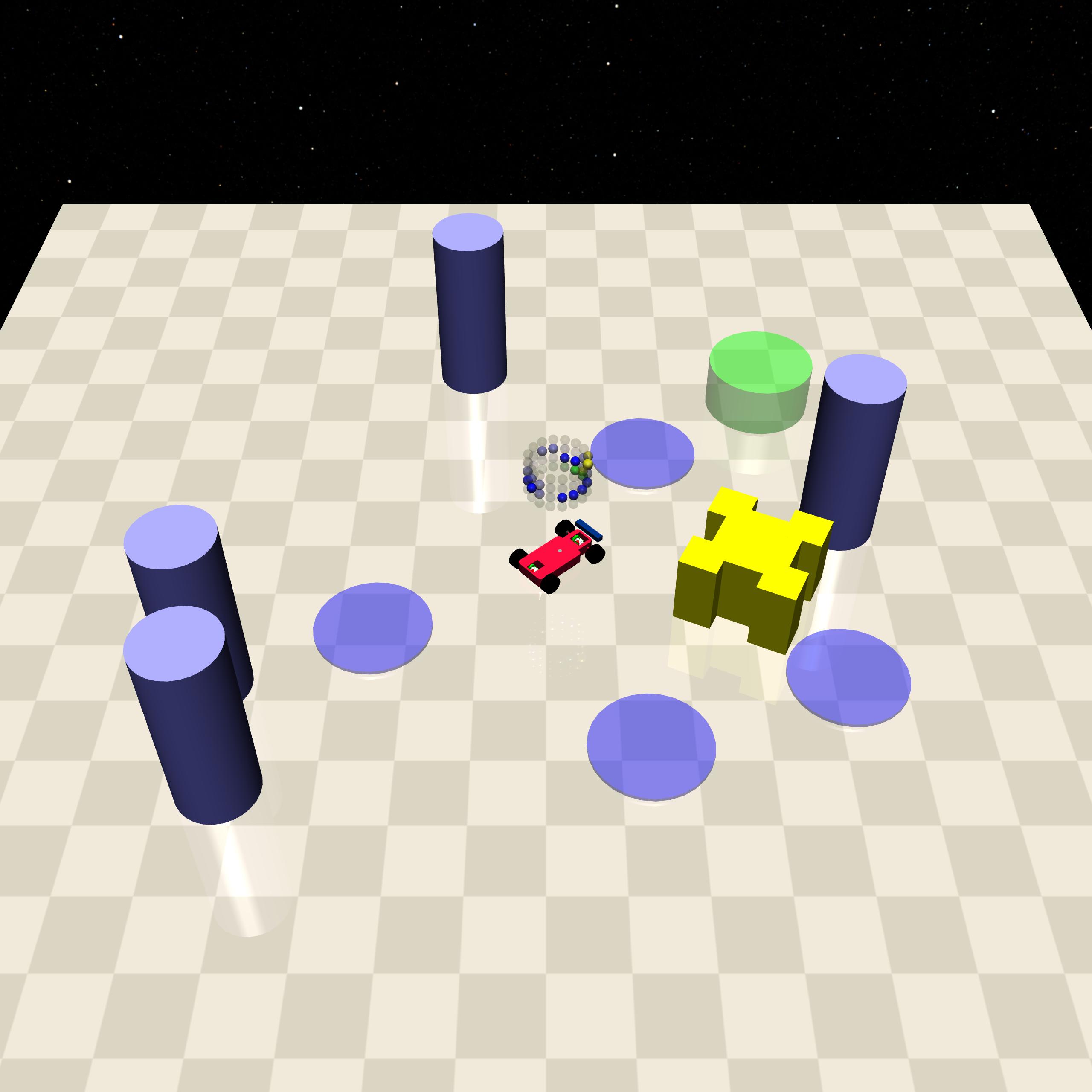}
        \caption{Push}
    \end{subfigure}
\hfill
    \begin{subfigure}{0.24\linewidth}
        \includegraphics[width=\linewidth]{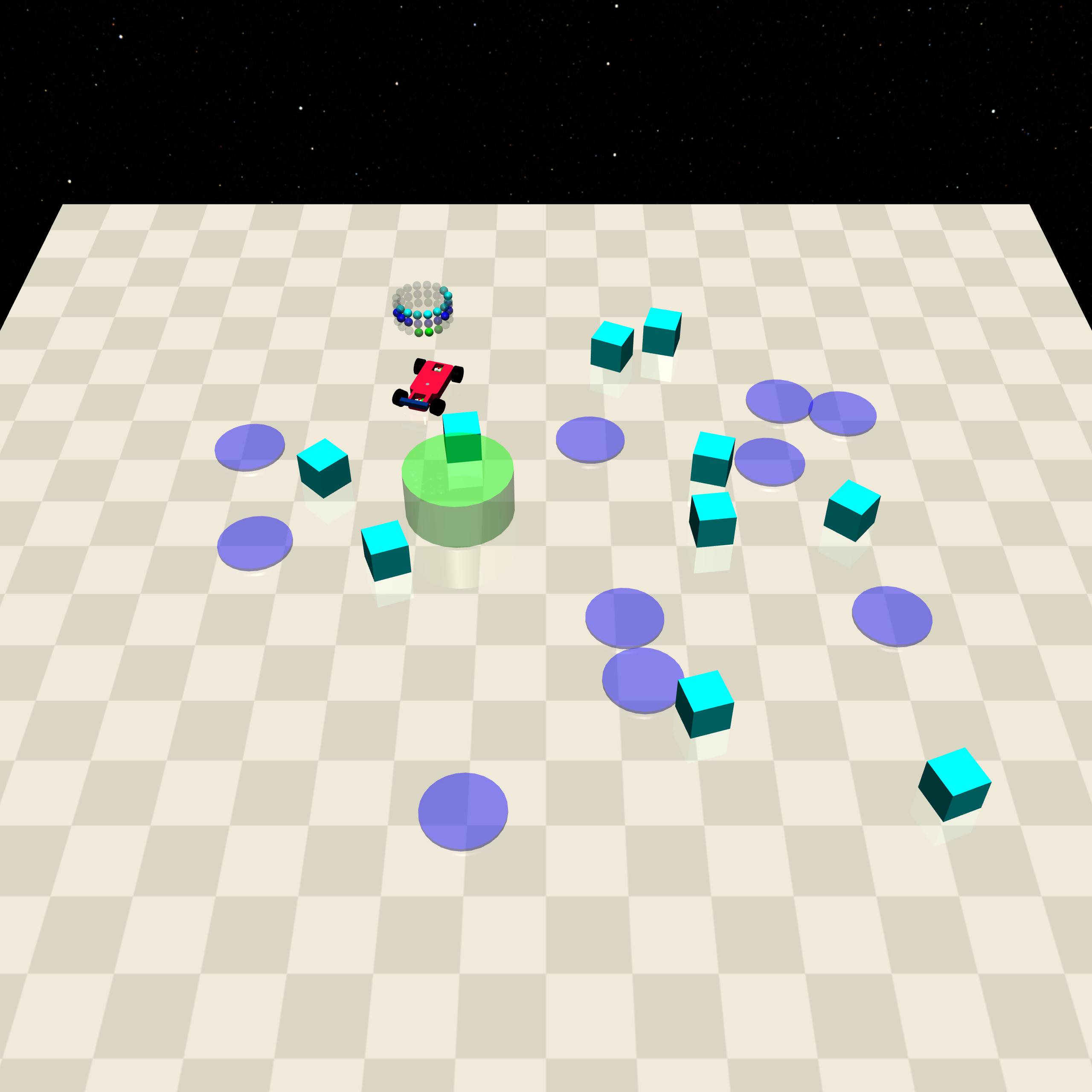}
        \caption{Goal}
    \end{subfigure}
    \hfill{}
    
    \caption{Safe navigation tasks of Safety Gymnasium \citep{ji2023safety} (images taken from \url{https://safety-gymnasium.readthedocs.io/en/latest/environments/safe_navigation.html})}
    \label{fig:sge}
\end{figure}
\begin{table}[t]
\caption{Overview of the Safety Gymnasium navigation tasks used in our experiments. The dimensionality refers to the state space for the Point and Car agents, respectively. (Recall that we always use level 2 of the tasks, which is the most difficult.)}
\label{tab:task}
\centering
\begin{tabularx}{\linewidth}{lXl}
\toprule
Task & Objective & dim. \\
\midrule
Circle & Circulate around the center of the circular area (green) while remaining within the boundaries (yellow), with rewards highest along the outermost circumference and increasing with speed. & 28/40\\
Push  & Push the box (yellow) to the goal location (green) while avoiding hazards and pillars (both blue).&76/88\\
Goal & Navigate to the goal (green) while avoiding hazards (blue) and vases (teal). & 60/72\\
\bottomrule
\end{tabularx}
\end{table}

Safe Navigation Tasks are described in \cref{tab:task} and visualised in \cref{fig:sge}.
Visual depictions and complete descriptions of the Safe Velocity tasks are available at \url{https://safety-gymnasium.readthedocs.io/en/latest/environments/safe_velocity.html}.

\paragraph{Notes on Agents.}
\begin{itemize}[noitemsep,topsep=0pt]
    \item \textbf{Point:} A 2D robot with two actuators for rotation and forward/backward movement.
    \item \textbf{Car:} A 3D robot with two independently driven wheels and a free rear wheel, requiring coordinated steering and movement.
    \item The Safe Velocity tasks are based on agents from MuJoCo \citep{todorov2012mujoco} locomotion tasks, which are extensively documented at \url{https://gymnasium.farama.org/environments/mujoco/}.
\end{itemize}

\subsubsection{Selection of Hyperparameters}
\label{app:hyper}
\label{app:ht}
\begin{table*}[t]
\centering
\caption{Hyperparameters for SB-TRPO (Ours) and baselines.}
\label{tab:hyperparams}
\resizebox{\textwidth}{!}{%
\begin{tabular}{lcccccccccccc}
\toprule
\textbf{Hyperparameter} & \multicolumn{2}{c}{\textbf{SB-TRPO (Ours)}} & \textbf{PPO-Lag} & \textbf{TRPO-Lag} & \textbf{CPO} & \textbf{CUP} & \textbf{CPPO-PID} & \textbf{FOCOPS} & \textbf{RCPO} & \textbf{PCPO} & \textbf{P3O} & \textbf{C-TRPO}  \\
 & \textbf{GAE} & \textbf{MC} &  &  &  &  &  &  &  &  &  &   \\
\midrule
Discount Factor $\gamma$ & 0.99 & 0.99 & 0.99 & 0.99 & 0.99 & 0.99 & 0.99 & 0.99 & 0.99 & 0.99 & 0.99 & 0.99  \\
Target KL & 0.01 & 0.01 & 0.02 & 0.01 & 0.01 & 0.02 & 0.02 & 0.02 & 0.01 & 0.01 & 0.02 & 0.01  \\
GAE $\lambda$ & 0.95 & -- & 0.95 & 0.95 & 0.95 & 0.95 & 0.95 & 0.95 & 0.95 & 0.95 & 0.95 & 0.95 \\
Timesteps per Epoch & 20000 & 20000 & 20000 & 20000 & 20000 & 20000 & 20000 & 20000 & 20000 & 20000 & 20000 & 20000 \\
Policy Hidden Layers & \{64, 64\} & \{64, 64\} & \{64, 64\} & \{64, 64\} & \{64, 64\} & \{64, 64\} & \{64, 64\} & \{64, 64\} & \{64, 64\} & \{64, 64\} & \{64, 64\} & \{64, 64\} \\
Policy Batch Size & 20000 & 20000 & 64 & 20000 & 20000 & 20000 & 64 & 64 & 20000 & 20000 & 64 & 20000  \\
Actor Learning Rate & -- & -- & 0.0003 & -- & -- & 0.0003 & 0.0003 & 0.0003 & -- & -- & 0.0003 & --  \\
Actor Optimiser & -- & -- & ADAM & -- & -- & ADAM & ADAM & ADAM & -- & -- & ADAM & --  \\
Actor Update Iterations & -- & -- & 40 & -- & -- & 40 & 40 & 40 & -- & -- & 40 & -- \\
Critic Hidden Layers & \{64, 64\} & \{64, 64\} & \{64, 64\} & \{64, 64\} & \{64, 64\} & \{64, 64\} & \{64, 64\} & \{64, 64\} & \{64, 64\} & \{64, 64\} & \{64, 64\} & \{64, 64\}\\
Critic Batch Size & 128 & -- & 64 & 128 & 128 & 64 & 64 & 64 & 128 & 128 & 64 & 256 \\
Critic Learning Rate & 0.001 & -- & 0.0003 & 0.001 & 0.001 & 0.0003 & 0.0003 & 0.0003 & 0.001 & 0.001 & 0.0003 & 0.001  \\
Critic Optimiser & ADAM & -- & ADAM & ADAM & ADAM & ADAM & ADAM & ADAM & ADAM & ADAM & ADAM & ADAM  \\
Critic Update Iterations & 10 & -- & 40 & 10 & 10 & 40 & 40 & 40 & 10 & 10 & 40 & 10 \\
Cost Limit & -- & -- & 0.00 & 0.00 & 0.00 & 0.00 & 0.00 & 0.00 & 0.00 & 0.00 & 0.00 & 0.00 \\
Clip Coefficient & -- & -- & 0.20 & -- & -- & 0.20 & -- & -- & -- & -- & 0.20 & -- \\
Conjugate Gradient (C.G.) Iterations & 50 & 50 & -- & 15 & 15 & -- & -- & -- & 15 & 15 & -- & 15  \\
C.G. Tikhonov Regularisation Coefficient & 0.02 & 0.02 & -- & 0.1 & 0.1 & -- & -- & -- & 0.1 & 0.1 & -- & 0.1  \\
Update Scaling Steps & 100 & 100 & -- & 15 & 15 & -- & -- & -- & -- & 200 & -- & 15 \\
Step Fraction & 0.80 & 0.80 & -- & 0.80 & 0.80 & -- & -- & -- & -- & 0.80 & -- & 0.80  \\
Lagrangian Initial Value & -- & -- & 0.001 & 0.001 & -- & 0.001 & 0.001 & 0.001 & -- & -- & -- & -- \\
Lagrangian Learning Rate & -- & -- & 0.035 & 0.035 & -- & 0.035 & -- & -- & -- & -- & -- & --  \\
Lagrangian Optimiser & -- & -- & ADAM & ADAM & -- & ADAM & -- & -- & -- & -- & -- & -- \\
CUP, FOCOPS $\lambda$ & -- & -- & -- & -- & -- & 0.95 & -- & 1.50 & -- & -- & -- & -- \\
CUP, FOCOPS $\nu$ & -- & -- & -- & -- & -- & 2.00 & -- & 2.00 & -- & -- & -- & --  \\
CPPO-PID $K_p$ & -- & -- & -- & -- & -- & -- & 0.1 & -- & -- & -- & -- & -- \\
CPPO-PID $K_i$ & -- & -- & -- & -- & -- & -- & 0.01 & -- & -- & -- & -- & --  \\
CPPO-PID $K_d$ & -- & -- & -- & -- & -- & -- & 0.1 & -- & -- & -- & -- & --  \\
P3O $\kappa$ & -- & -- & -- & -- & -- & -- & -- & -- & -- & -- & 10.0 & --  \\
C-TRPO $\beta$ & -- & -- & -- & -- & -- & -- & -- & -- & -- & -- & -- & 1.0  \\
\alg $\fimp$ & 0.75 & 0.75 & -- & -- & -- & -- & -- & -- & -- & -- & -- & --  \\
\bottomrule
\end{tabular}
}
\end{table*}

\cref{tab:hyperparams} lists all hyperparameters. 
For the baselines, we use the official SafePO defaults \citep{ji2023safety}, 
setting the cost limit to $0$ to target almost-sure safety. 
For \alg, most hyperparameters are inherited from related methods 
(TRPO-Lagrangian and CPO), and critic-related settings are kept identical in critic-inclusive runs. 
In runs without critics we use Monte Carlo returns, i.e.\ GAE with $\lambda=1$.
The only novel parameter is the safety bias $\fimp$, for which no established default exists. 
We therefore performed a light grid search over $[0.6,0.9]$ as shown in \cref{fig:fimp_scatter} in \cref{sec:ablation} and \cref{fig:fimp_scatter2} in \cref{app:fimp}. Different choices of $\fimp$ trace out an approximate linear 
reward-cost Pareto frontier, while baselines typically underperform relative to this frontier. 
For the main results (cf.~\cref{tab:main-results}), we chose \emph{$\fimp\defeq 0.75$ across all tasks}, 
highlighting the robustness of our approach.

\subsubsection{Formal Definition of Safety Metrics}
\label{app:metric}

\emph{Safety probability} is the fraction of episodes completed without any safety violation, i.e.\ with zero cost;
\emph{safe reward} is the average return over all episodes, with returns from episodes with any safety violation set to $0$:
\begin{align}
\label{eq:defmetric}
    \text{safety probability}&\defeq \frac{1}{N}\sum_{i=1}^N \mathbf{1}[C_i = 0]
&
\text{safe reward}&\defeq \frac{1}{N}\sum_{i=1}^N R_i\cdot\mathbf{1}[C_i = 0]
\end{align}
where $R_i$ and $C_i$ are the total return and cost of episode $i$.

\subsubsection{Reproduction of Results}
\label{app:rep}
All experiments were run with \texttt{Python~3.8}. To reproduce our results, first install Safety-Gymnasium as documented at \url{https://github.com/PKU-Alignment/safety-gymnasium#installation}, and then extract the supplementary code provided with this submission. Inside the extracted folder, install SafePO \citep{ji2023safety} integrated with our approach by running:
\begin{verbatim}
pip install -e .
\end{verbatim}
Experiments can be reproduced in two ways. To train a single baseline (e.g., PPO-Lagrangian) with a specific seed (e.g., 2000) on a given environment (e.g., SafetyPointGoal2-v0), navigate to the \texttt{single{\textunderscore}agent} directory and run:
\begin{verbatim}
python ppo_lag.py --task SafetyPointGoal2-v0 --seed 2000 --cost-limit 0
\end{verbatim}
Similarly, for our method, \alg, at a specific safety bias (e.g., $\fimp=0.8$), use:
\begin{verbatim}
python sb-trpo.py --task SafetyPointGoal2-v0 --seed 2000 --beta 0.8
\end{verbatim}
Additional options can be found in \texttt{utils/config.py} under \texttt{single{\textunderscore}agent{\textunderscore}args}.

For large-scale benchmarking, multiple algorithms can be launched in parallel with a single command:
\begin{verbatim}
python benchmark.py
--tasks SafetyHopperVelocity-v1 SafetySwimmerVelocity-v1
--algo ppo_lag trpo_lag cpo cup sb-trpo 
--workers 25 --num-envs 20 --steps-per-epoch 20000
--total-steps 20000000 --num-seeds 5
--cost-limit 0 --beta 0.75
\end{verbatim}
This command runs all baselines considered in the paper at a cost limit of 0, alongside our method with $\fimp=0.75$, using 25 training processes in parallel with 20 vectorized environments per process and 5 random seeds. Further options are documented in \texttt{benchmark.py}.

Finally, our ablations are implemented as \texttt{sb-trpo{\textunderscore}rcritic} (including a reward critic) and \texttt{sb-trpo{\textunderscore}critic} (including both reward and cost critics).

\subsection{Supplementary Materials for \cref{sec:res}}
\label{app:res}

\begin{figure*}[t!]
    \centering
    \includegraphics[width=\textwidth]{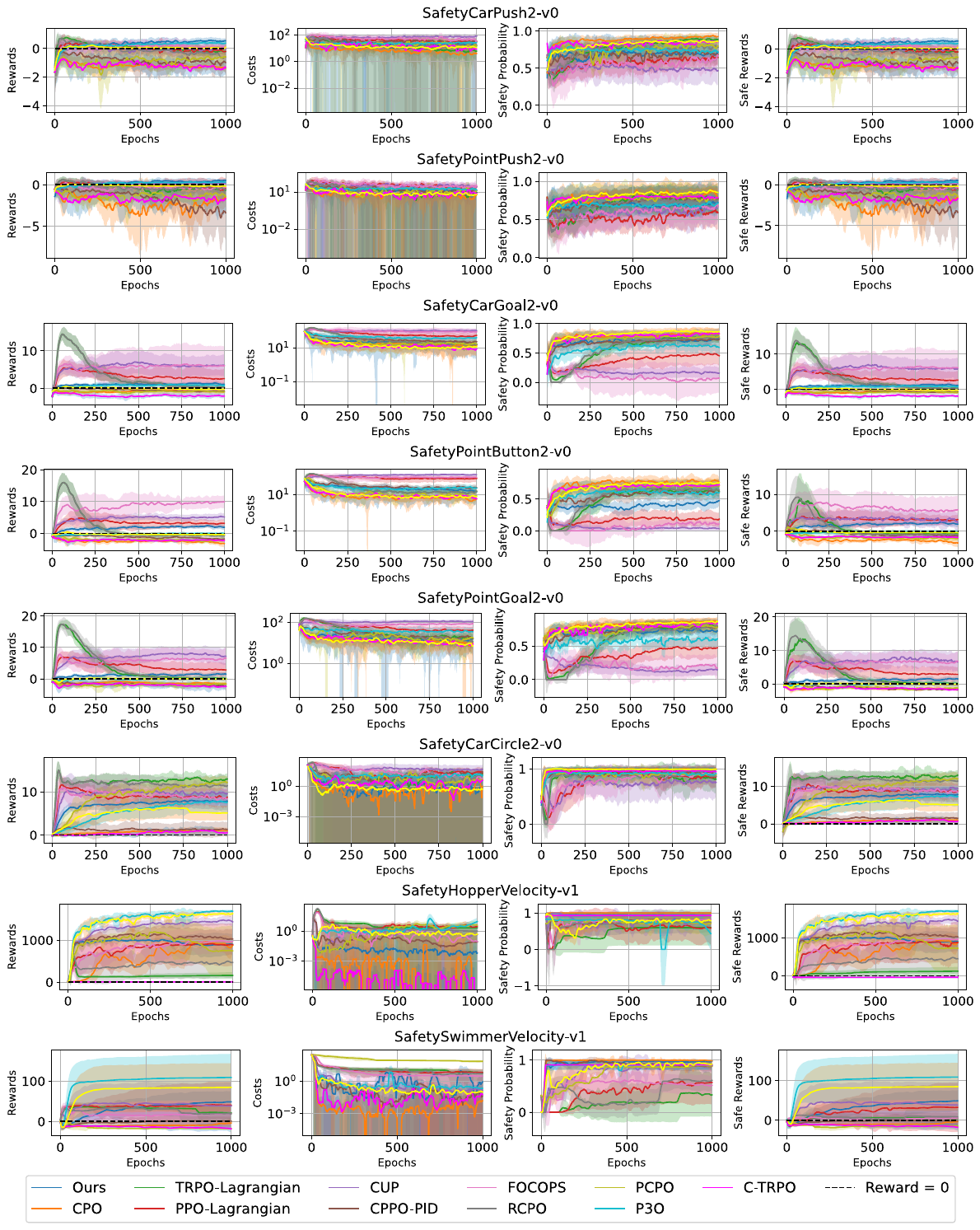}
    \caption{Training curves}
    \label{fig:part_1}
\end{figure*}
\begin{figure*}[h]
    \centering
    \includegraphics[width=\textwidth]{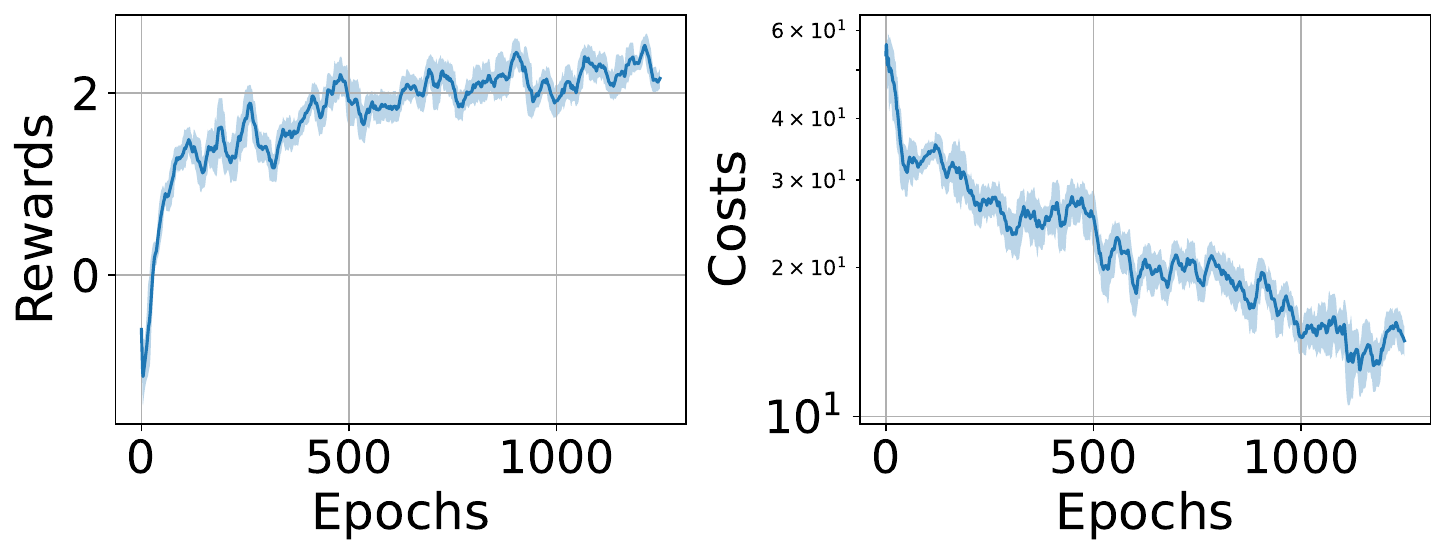}
    \caption{Point Button: training longer, doing better.}
    \label{fig:pb_long}
\end{figure*}

The training curves, which are smoothed with averaging over the past 20 epochs, are presented in \cref{fig:part_1}.

\begin{figure}
        \centering
        \includegraphics[width=\linewidth]{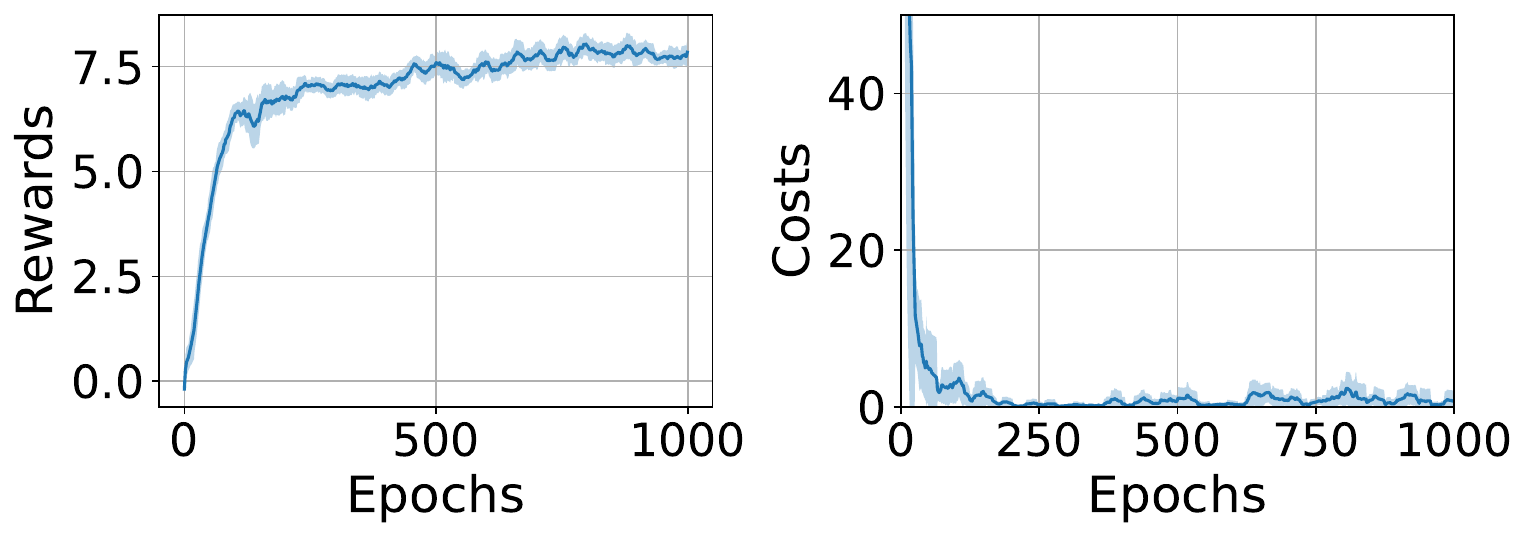}
        \caption{Car Circle: staying safe and improving reward}
        \label{fig:cost_imp}
    \end{figure}
    
Beyond final performance, our approach also  exhibits robustness in practice:  
temporary increases in cost are typically corrected quickly whilst improving rewards overall (see \cref{fig:cost_imp}).
Moreover, longer training can continue to improve both reward and feasibility for our method, while the baselines plateau (cf.~\cref{fig:pb_long,fig:part_1} in \cref{app:res}). \dw{NB moved here}

Additionally, we train longer on Point Button to confirm further feasibility and reward improvements.
This is presented in \cref{fig:pb_long}. As anticipated, the improvements are non-diminishing even for a 125\% longer training run.


We compare the compute costs of our approach with the baselines w.r.t. the update times per epoch on a single, non-parallelized training run in \cref{tab:time}. This demonstrates that training critics is expensive across all critic-integrated approaches.

\subsubsection{Qualitative Behaviour and Videos}
We compare the qualitative behaviours of policies obtained from our approach with the baselines using recorded videos (provided in the supplementary material) across Car Circle, Car Goal, and Swimmer Velocity. The results in \cref{tab:main-results} are strongly corroborated by the behaviours observed in these videos.


\textbf{Car Circle.} CPO and C-TRPO fail to fully associate circling movement with reward acquisition, resulting in back-and-forth movements near the circle’s centre and partial rewards. Consequently, these policies rarely violate constraints. P3O, CUP, and PPO-Lag learn to circle effectively but struggle to remain within boundaries, incurring costs. TRPO-Lag and \alg achieve both effective circling and boundary adherence.

\textbf{Car Goal.} CPO, and C-TRPO display low reward affinity, leading to overconservative behaviours such as avoiding the interactive region. TRPO-Lag is unproductive and does not avoid hazards reliably. CUP and PPO-Lag are highly reward-driven but reckless with hazards. P3O begins to learn safe navigation around the interactive region, though occasional recklessness remains. SB-TRPO achieves productive behaviour while intelligently avoiding hazards, e.g., by slowly steering around obstacles or moving strategically within the interactive region.

\textbf{Swimmer Velocity.} CPO and C-TRPO sometimes move backward or remain unproductively in place. TRPO-Lag and PPO-Lag move forward but engage in unsafe behaviours. CUP and \alg are productive in moving forward while being less unsafe. P3O is the most effective, moving forward consistently while exhibiting the fewest unsafe behaviours.

\subsubsection{Deviations from Theory in Training Curves} 
While our theoretical results (\cref{thm:perf}) guarantee consistent improvement in both reward and cost for sufficiently small steps, practical training exhibits deviations due to finite sample estimates and the sparsity of the cost signal in some tasks. We summarise task-specific observations:

\begin{itemize}[noitemsep,topsep=0pt]
    \item In \textbf{Swimmer Velocity}, some update steps may not immediately reduce incurred costs due to estimation noise, which explains delays in achieving perfect feasibility.
    \item In \textbf{Car Circle}, policies that previously attained near-zero cost sometimes return temporarily to infeasible behaviour. This arises because the update rule does not explicitly encode ``distance'' from infeasible regions within the feasible set. Nonetheless, policies eventually return to feasibility thanks to the guarantees of our approach.
    \item In \textbf{Hopper Velocity}, after quickly achieving high reward, performance shows a slow decline accompanied by gradual improvements in cost. This reflects the prioritisation of feasibility improvement over reward during some updates, without drastically compromising either objective.
    \item In \textbf{navigation tasks}, especially \textbf{Goal and Push}, the cost signal is sparse: a cost of 1 occurs only if a hazard is touched. This sparsity makes perfect safety more difficult to achieve, as most updates provide limited feedback about near-misses or almost-unsafe trajectories. Despite this, SB-TRPO steadily reduces unsafe visits and substantially outperforms baselines.
\end{itemize}

Overall, these deviations are a natural consequence of finite sample estimates, quadratic approximations, and sparse cost signals. In spite of this, SB-TRPO consistently improves reward and safety together, achieving higher returns and lower violations than all considered baselines.

\begin{table}[t]
    \centering
    \caption{Update times per epoch on Point Goal on a single, non-parallelized training run, based on the hyperparameters in \cref{tab:hyperparams}. Mean $\pm$ STDEV is reported, rounded to 2 decimal places. The least update time consumed per epoch is in \textbf{bold}. Note that PPO-Lagrangian, in contrast to TRPO-Lagrangian, has a budget of up to 40 gradient steps per epoch (as per \cref{tab:hyperparams}).}
    \label{tab:example}
    \resizebox{\textwidth}{!}{%
    \begin{tabular}{lccccccccccccc}
        \toprule
         & \alg (MC) & 
         \alg (GAE) & TRPO-Lag & CPO & PPO-Lag & CUP & CPPO-PID & FOCOPS & RCPO & PCPO & P3O & C-TRPO  \\
        \midrule
        Time$^\downarrow$ & $\mathbf{0.26\pm0.08}$ & 
        $8.12\pm0.95$ & $5.85\pm0.79$ & $6.59\pm1.96$ & $69.83\pm8.29$ & $84.60\pm15.54$ & $88.32\pm18.82$ & $53.54\pm19.34$ & $4.40\pm0.35$ & $5.74\pm1.19$ & $78.05\pm20.14$ & $3.32\pm0.16$ \\
        \bottomrule
    \end{tabular}
    \label{tab:time}
    }
\end{table}

\subsubsection{Supplementary Materials for \cref{sec:align}}
\label{app:align}

\begin{figure}[t!]
    \centering
    \hfill
    \begin{subfigure}{0.49\linewidth}
        \includegraphics[width=\linewidth]{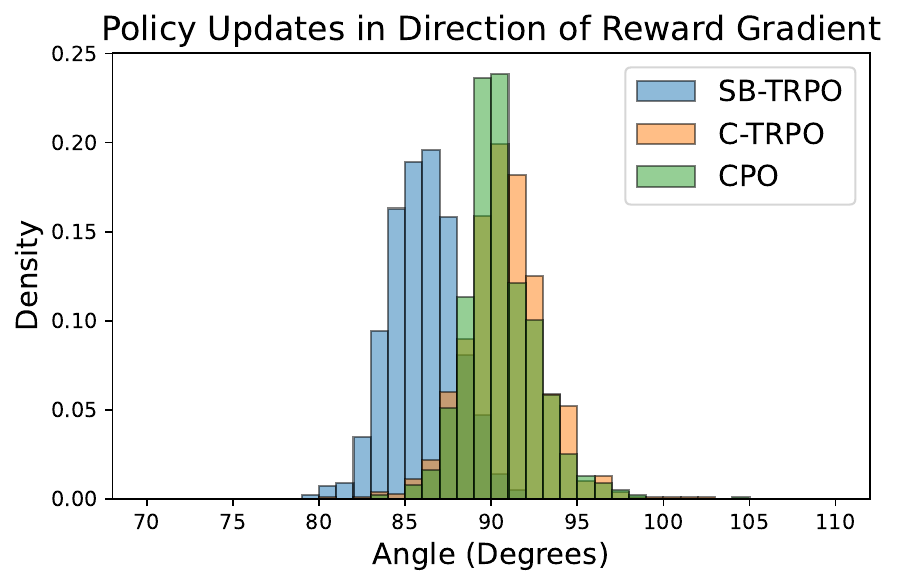}
        \caption{Point Button}
    \end{subfigure}    
    \hfill
    \begin{subfigure}{0.49\linewidth}
    \includegraphics[width=\linewidth]{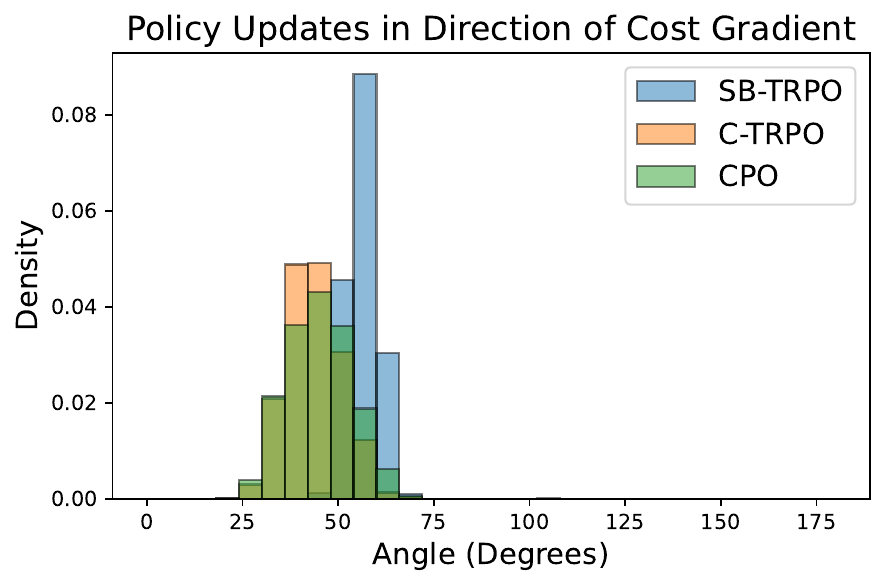}
        \caption{Point Button}
    \end{subfigure}    
    \hfill{}
    \caption{Angles between policy updates and reward/cost gradients for Point Button}
    \label{fig:histc}    
\end{figure}
Angles between policy updates and cost gradients are plotted in \cref{fig:histc}.

\subsubsection{Supplementary Materials for \cref{sec:clim}}
\label{app:clim}
We study the performance of baselines under a positive cost threshold of 15 (see \cref{tab:results-15}).

\begin{table*}[t!]
    \centering
    \caption{Selected baselines with a relaxed cost limit of 15 (using the same highlighting as in \cref{tab:main-results}).}
    \resizebox{\textwidth}{!}{%
    \begin{tabular}{lccccccccc}
\toprule
\multirow{2}{*}{Method} & Cost & \multicolumn{2}{c}{Car Circle} & \multicolumn{2}{c}{Point Push} & \multicolumn{2}{c}{Point Goal} &  \multicolumn{2}{c}{Swimmer Velocity} \\
\cmidrule(lr){3-10}
& limit & Safe Rew. & Safe Prob. & Safe Rew. & Safe Prob. & Safe Rew. & Safe Prob. & Safe Rew. & Safe Prob. \\
\midrule
Ours & -- & $7.5\pm1.5$ & $\boldsymbol{0.99\pm0.01}$  & $0.33\pm0.18$ & $\boldsymbol{0.79\pm0.08}$ & $\boldsymbol{1.6\pm0.5}$ & $0.68\pm0.07$ & $\boldsymbol{48\pm26}$ & $\boldsymbol{0.98\pm0.02}$ \\
\midrule TRPO-Lag & 0 & $\boldsymbol{11\pm4}$ & $0.81\pm0.26$  & \color{negrew}$-0.49\pm0.45$ & \color{negrew}$0.78\pm0.16$ & $0.09\pm0.29$ & $\boldsymbol{0.79\pm0.06}$ & \color{dom}$\mathit{1.3\pm5.1}$ & \color{dom}$\mathit{0.33\pm0.42}$ \\
TRPO-Lag & 15 & \color{dom}$\mathit{6.8\pm1.5}$ & \color{dom}$\mathit{0.44\pm0.09}$ & \color{dom}$0.29\pm0.28$ & \color{dom}$0.64\pm0.12$ & \color{negrew}$-0.18\pm0.53$ & \color{negrew}$0.82\pm0.08$  & \color{dom}$\mathit{0.0\pm0.0}$ & \color{dom}$\mathit{0.0\pm0.0}$ \\
\midrule CPO & 0 & \color{dom}$0.96\pm0.76$ & \color{dom}$\boldsymbol{0.99\pm0.03}$ & \color{negrew}$-1.2\pm0.9$ & \color{negrew}$0.77\pm0.19$ & \color{negrew}$-1.3\pm0.4$ & \color{negrew}$0.86\pm0.05$  & \color{negrew}$-3.8\pm9.3$ & \color{negrew}$0.99\pm0.03$ \\
CPO & 15 & \color{dom}$8.3\pm1.6$ & \color{dom}$0.54\pm0.11$ & \color{dom}$0.24\pm0.14$ & \color{dom}$0.61\pm0.07$ & \color{dom}$0.29\pm0.67$ & \color{dom}$0.69\pm0.11$  & \color{dom}$\mathit{0.0014\pm0.0098}$ & \color{dom}$\mathit{0.001\pm0.0052}$ \\
\midrule RCPO & 0 & $10\pm3$ & $0.9\pm0.18$ & \color{negrew}$-0.48\pm0.45$ & \color{negrew}$0.84\pm0.08$ & \color{negrew}$-0.012\pm0.366$ & \color{negrew}$0.85\pm0.06$  & \color{dom}$7.0\pm12.3$ & \color{dom}$0.67\pm0.25$ \\
RCPO & 15 & \color{dom}$9.4\pm3.6$ & \color{dom}$0.61\pm0.23$ & $0.43\pm0.19$ & $0.6\pm0.1$ & $1.0\pm0.5$ & $0.71\pm0.07$ & \color{dom}$\mathit{0.0\pm0.0}$ & \color{dom}$\mathit{0.0\pm0.0}$ \\
\midrule C-TRPO & 0 & \color{dom}$0.54\pm0.5$ & \color{dom}$0.97\pm0.05$ & \color{negrew}$-1.4\pm1.2$ & \color{negrew}$0.83\pm0.09$ & \color{negrew}$-2.2\pm1.1$ & \color{negrew}$0.88\pm0.06$  & \color{negrew}$-18\pm8$ & \color{negrew}$0.97\pm0.03$ \\
C-TRPO & 15 & \color{dom}$\boldsymbol{11\pm1}$ & \color{dom}$0.69\pm0.08$ & $\boldsymbol{0.48\pm0.3}$ & $0.58\pm0.08$ & \color{dom}$1.2\pm0.5$ & \color{dom}$0.66\pm0.08$  & \color{dom}$\mathit{0.05\pm0.16}$ & \color{dom}$\mathit{0.0092\pm0.0223}$ \\
\bottomrule
    \end{tabular}
    \label{tab:results-15}
    }
\end{table*}

\subsection{Supplementary Materials for \cref{sec:ablation}}
\label{app:ablation}

\subsubsection{Effect of Safety Bias $\fimp$}
\label{app:fimp}
\begin{figure*}
\begin{subfigure}{\linewidth}
    \centering
    
        \includegraphics[width=\linewidth]{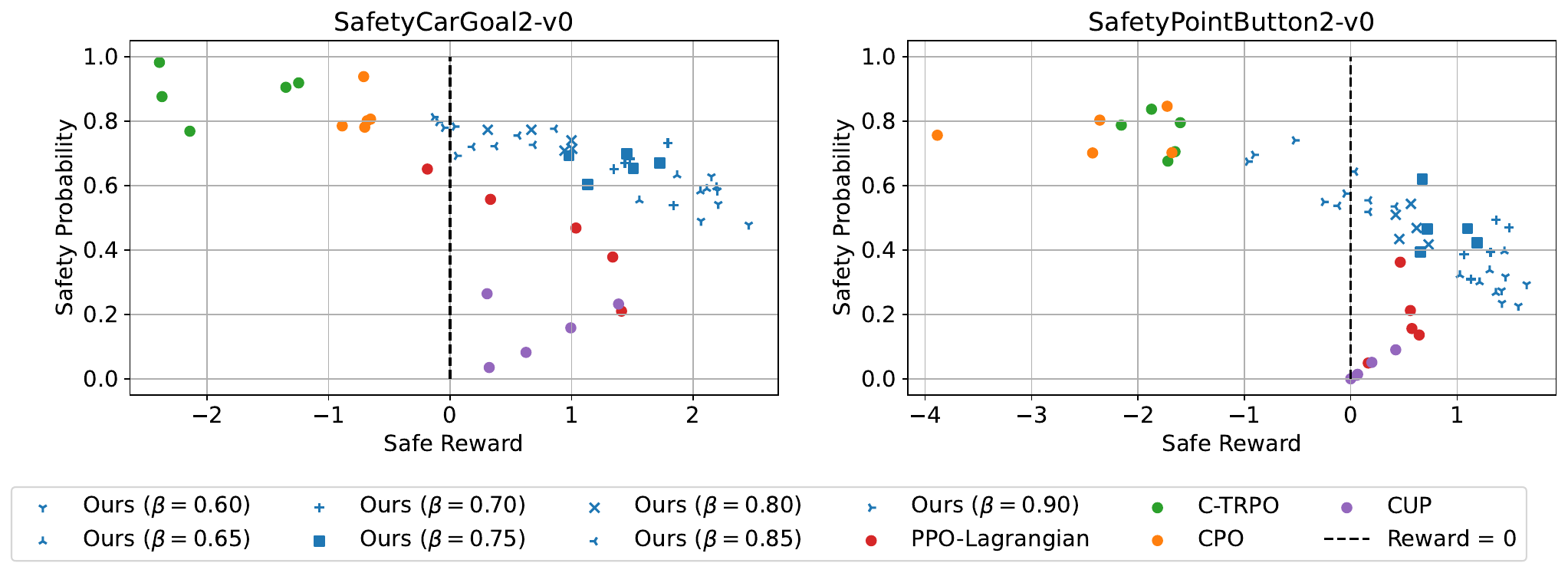}
    \caption{Safe reward and safety probability}
   \label{fig:fimp_scatter1} 
\end{subfigure}
\begin{subfigure}{\linewidth}
    \centering
    \includegraphics[width=\textwidth]{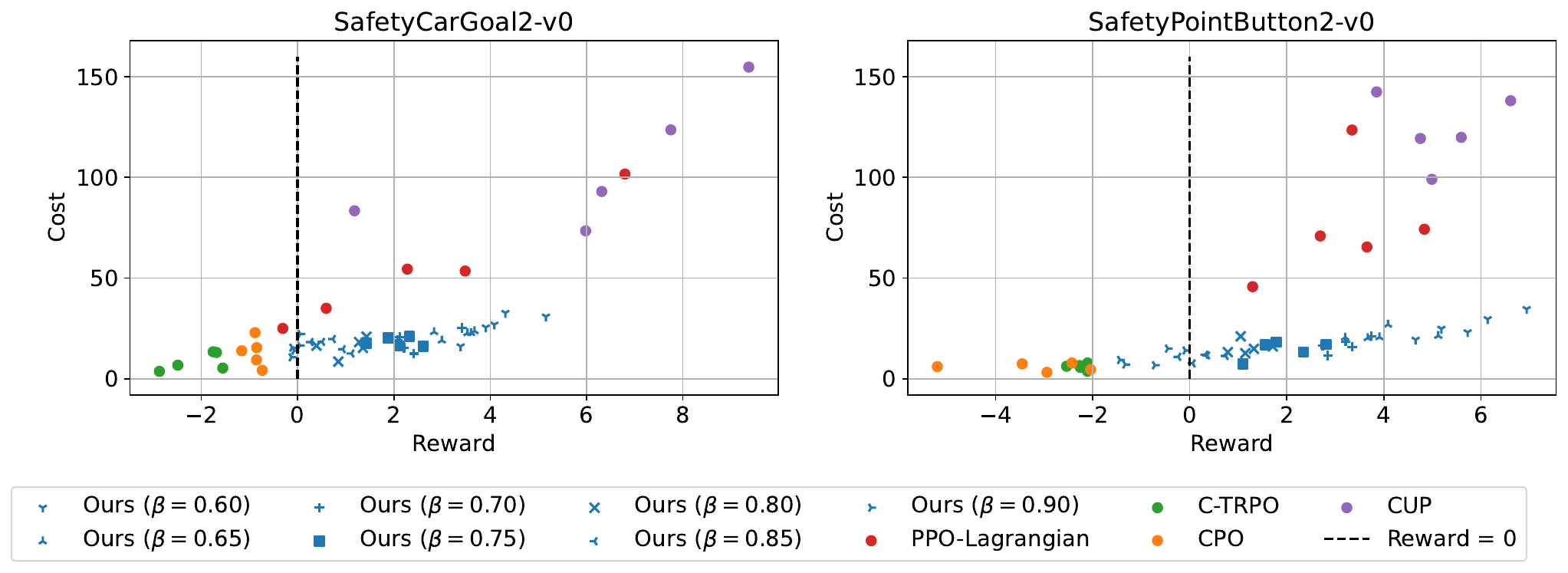}
    \caption{Reward and cost}
    \label{fig:fimp_scatter2}
\end{subfigure}
    \caption{Ablation study of safety bias $\fimp$}
    \label{fig:fimp_scatter} 
\end{figure*}
We present scatter plots for our approach compared to selected baselines in \cref{fig:fimp_scatter}. 
\alg consistently achieves the best balance of performance and feasibility across all benchmark tasks, demonstrating robustness to the choice of $\fimp$.

\subsubsection{Advantage Estimation}
\label{app:critic}

\begin{table*}[t!]
    \centering
    \caption{Effect of advantage estimation on performance}
    \resizebox{\textwidth}{!}{%
    \begin{tabular}{lccccccccc}
    \toprule
        \multirow{2}{*}{Advantage} & \multirow{2}{*}{$\fimp$} & \multicolumn{2}{c}{Car Goal} & \multicolumn{2}{c}{Point Goal} & \multicolumn{2}{c}{Car Circle} & \multicolumn{2}{c}{Swimmer Velocity} \\
\cmidrule(lr){3-10}
&  & Safe Rew.$^\uparrow$ & Safe Prob.$^\uparrow$ & Safe Rew.$^\uparrow$ & Safe Prob.$^\uparrow$ & Safe Rew.$^\uparrow$ & Safe Prob.$^\uparrow$ & Safe Rew.$^\uparrow$ & Safe Prob.$^\uparrow$ \\
\midrule

 MC & $0.75$ & ${1.4\pm0.4}$ & $0.66\pm0.08$ & ${1.6\pm0.5}$ & $0.68\pm0.07$ & $7.5\pm1.5$ & ${0.99\pm0.01}$ & $48\pm26$ & ${0.98\pm0.02}$ \\
 GAE & $0.75$ & $5.9\pm0.9$ & $0.40\pm0.07$ & $5.0\pm1.3$ & $0.51\pm0.09$ & $10.2\pm2.2$ & $0.96\pm0.07$ & $0.67\pm8.1$ & $1.0\pm0.0$\\
 GAE & $0.8$ & $2.3\pm	0.35$&	$0.62\pm	0.066$ & $1.7\pm	0.92$&	$0.69\pm	0.098$& $	8\pm	3.1$&	$0.99\pm	0.017$ & 	$18\pm	35$&	$0.98\pm	0.031$ \\
\bottomrule
    \end{tabular}
    \label{tab:critic}
    }
\end{table*}

The values of metrics regarding the ablation study about the effect of critics are provided in 
\cref{tab:critic,tab:time}.

\section{Approximations}

\begin{figure}[t!]
    \centering
    \begin{subfigure}{0.49\linewidth}
    \includegraphics[width=\linewidth]{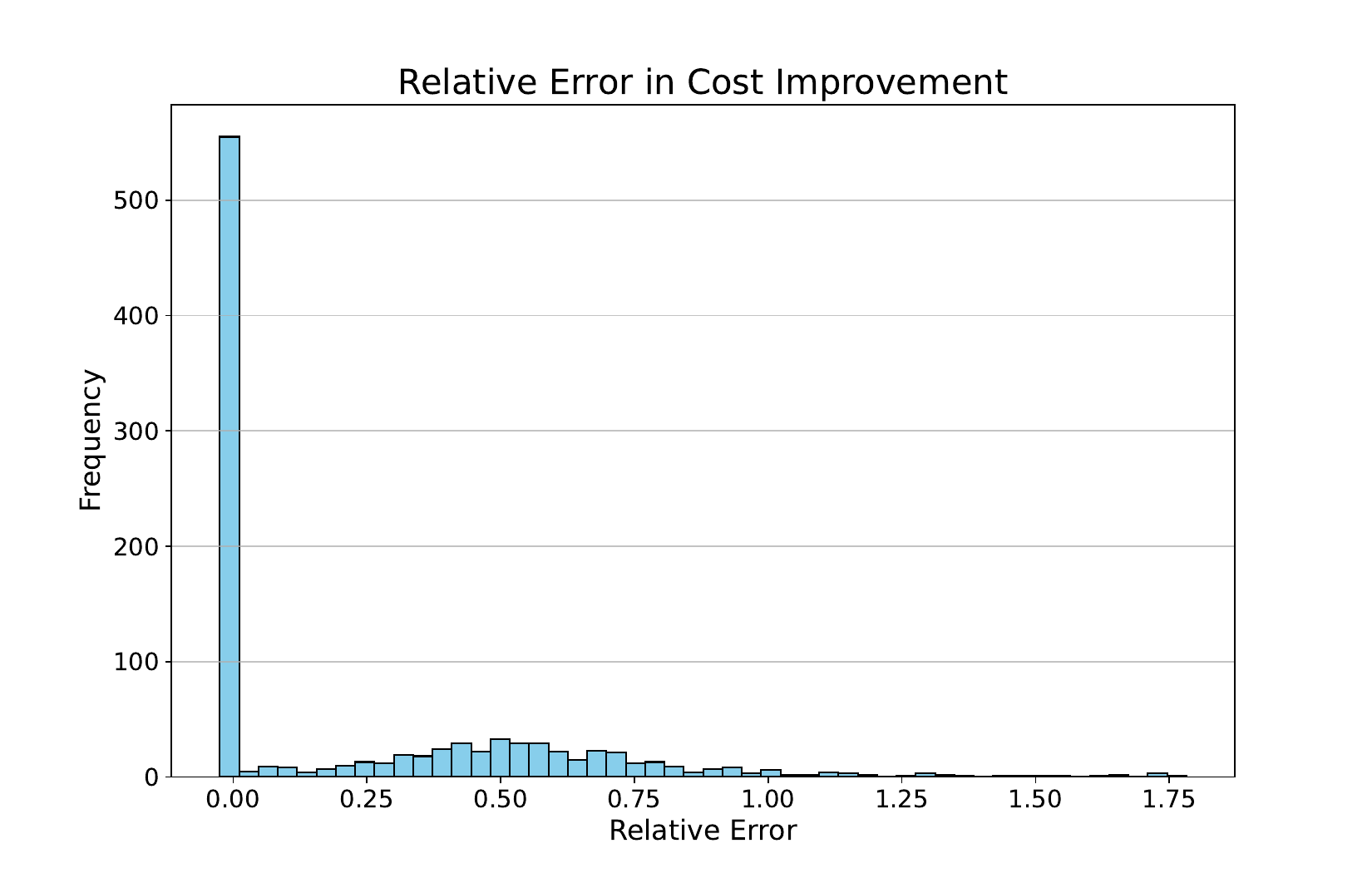}
    \caption{Cost}
    \end{subfigure}
    \begin{subfigure}{0.49\linewidth}
    \includegraphics[width=\linewidth]{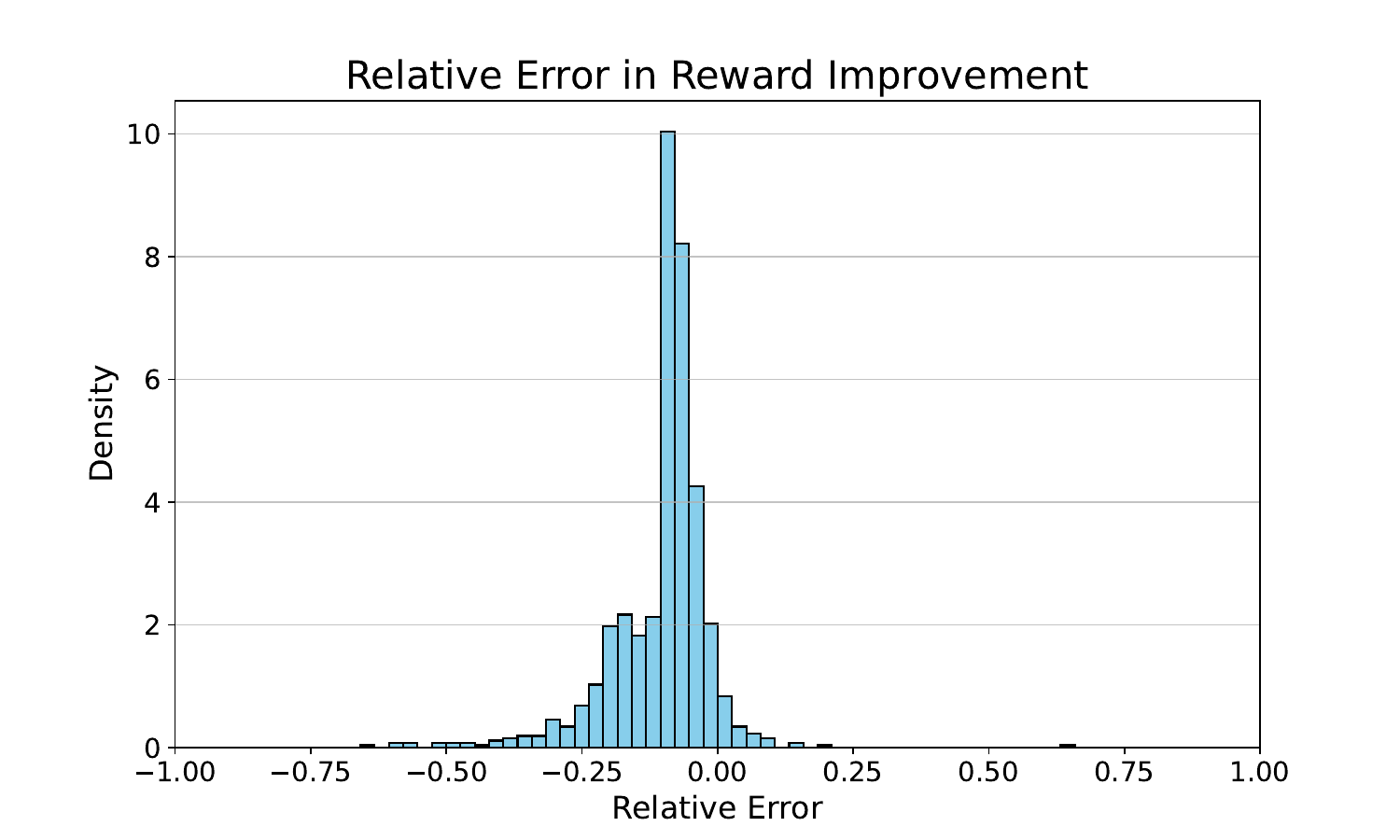}
        \caption{Reward}
    \end{subfigure}    
    \begin{subfigure}{0.49\linewidth}
    \includegraphics[width=\linewidth]{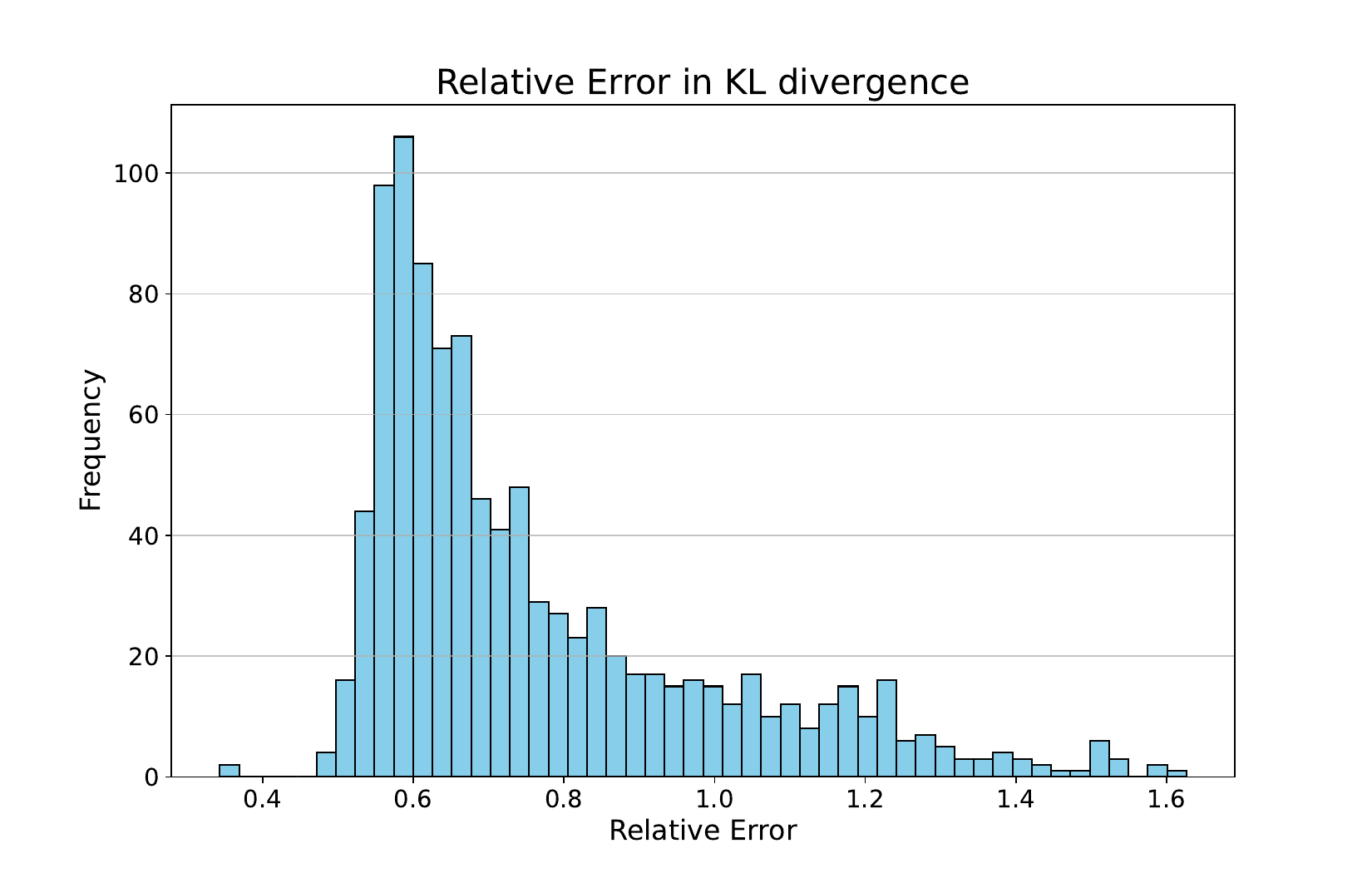}
        \caption{KL divergence}
    \end{subfigure}    
    \caption{Histograms of relative approximation errors ($\frac{\widehat x-x}x$).}
    \label{fig:histapprox}    
\end{figure}

As in most policy optimisation methods, we begin from an idealised \pref{eq:general} with theoretical guarantees (\cref{thm:optideal}) and introduce a sequence of controlled approximations to obtain a tractable algorithm.

The approximations fall into two categories:
\begin{itemize}[noitemsep,topsep=0pt]
    \item Our innovation: solving \pref{eq:approx} via a convex combination of reward and cost directions
    \item Standard trust-region approximations required for tractability
\end{itemize}

\paragraph{Our Innovation: Convex Combination.}
Instead of solving \pref{eq:approx} exactly, we approximate its solution via a convex combination of the reward and cost natural gradient directions (\cref{eq:approxc,eq:approxr}).

We provide the following guarantees:
\begin{enumerate}[noitemsep,topsep=0pt]
    \item The coefficient $\conv$ defined in \cref{eq:conv} ensures feasibility of \pref{eq:approx}, with a linearised decrease in cost of at least $\epsc$ (see \cref{lem:kl,lem:conv}).
    \item For sufficiently small step sizes this translates into a decrease in (\cref{thm:perf}).
    \item \cref{thm:perf} further ensures reward improvement when gradients are not adversarially aligned (angle $\leq 90^\circ$).
\end{enumerate}

Thus, while not necessarily optimal, the update is \emph{provably safe and improving under mild conditions}, which is the key requirement in the constrained setting.

\paragraph{Standard Practical Approximations.}
The remaining approximations follow established trust-region practices:

\begin{itemize}[noitemsep,topsep=0pt]
    \item Surrogate Objective: We replace returns with importance-sampling-based surrogates (Update 2). The policy improvement bound, coupled with the KL trust-region constraint, ensures a reliable local approximation.
    \item Quadratic Approximation: \pref{eq:ideal} is approximated using gradients and the Fisher information matrix for the KL divergence in \pref{eq:approx}.
    \item Numerical Efficiency: Natural policy gradient directions are computed approximately using Conjugate Gradient.
\end{itemize}

\paragraph{Line Search and Empirical Safeguards.}
To bridge the gap between theory and practice, we use backtracking line search (l.~12 in \cref{alg}), enforcing the KL constraint empirically to ensure surrogate accuracy and verifying the empirical reduction in surrogate cost loss.

\paragraph{Gradient Estimation.}
Finally, we note that gradient estimation is a fundamental challenge in RL. This issue is faced by all policy gradient and trust-region methods, and is not specific to our approach.

\paragraph{Empirical Validation.}
We empirically evaluate approximation accuracy by comparing predicted (first-/second-order) vs. actual surrogate improvements on the same samples. We include histograms \cref{fig:histapprox} of relative errors for cost/reward improvement and KL divergence.

We observe the following findings:
\begin{itemize}[noitemsep,topsep=0pt]
    \item Cost improvement is predicted accurately, supporting the practical validity of Thm. 4.2.
\item Reward improvement is under-predicted, indicating a conservative bias.
\item KL is over-approximated, resulting in stricter-than-necessary trust regions.
\end{itemize}

\paragraph{Conclusion.}
SB-TRPO builds on trusted TRPO/CPO approximations while introducing a convex combination update with performance improvement guarantees (\cref{thm:perf}). Empirical validation shows these approximations are conservative, and combined with line search, they ensure robustness even with gradient estimates.

\end{document}